%% file: main.tex
\documentclass{article}

\usepackage[final,main]{neurips_2026}

\usepackage[utf8]{inputenc}
\usepackage[T1]{fontenc}
\usepackage{hyperref}
\usepackage{url}
\usepackage{booktabs}
\usepackage{amsfonts}
\usepackage{nicefrac}
\usepackage{microtype}
\usepackage{xcolor}
\usepackage{enumitem}
\usepackage{algorithm}
\usepackage{algorithmic}
\usepackage{graphicx}
\usepackage{subcaption}

\usepackage{dsfont}
\usepackage{amsmath}
\usepackage{amssymb}
\usepackage{mathtools}
\usepackage{amsthm}
\usepackage{tabularx}
\usepackage[capitalize,noabbrev]{cleveref}
\usepackage{bm}

\newcommand{\appref}[1]{Appendix~\ref{#1}}

\theoremstyle{plain}
\newtheorem{theorem}{Theorem}
\newtheorem{proposition}{Proposition}
\newtheorem{lemma}{Lemma}
\newtheorem{corollary}{Corollary}
\theoremstyle{definition}
\newtheorem{definition}{Definition}
\newtheorem{assumption}{Assumption}
\crefname{assumption}{Assumption}{Assumptions}
\newtheorem{example}{Example}
\theoremstyle{remark}
\newtheorem{remark}{Remark}

\newcommand{\R}{\mathbb{R}}
\newcommand{\E}{\mathbb{E}}
\newcommand{\cL}{\mathcal{L}}
\newcommand{\cF}{\mathcal{F}}

\newcommand{\Var}{\mathrm{Var}}

\newcommand{\loss}{\ell}
\newcommand{\eapp}{\varepsilon_{\mathrm{app}}}
\newcommand{\Lloss}{L_\ell}

\usepackage{todonotes}
\usepackage{wrapfig}

\title{Fast Statistical Rates for Flow Matching without Curvature Assumptions}

\author{%
  Hari Krishna Sahoo \\
  Purdue University\\
  West Lafayette, IN 47906\\
  \texttt{hsahoo@purdue.edu}
  \And
  Mudit Gaur \\
  Purdue University\\
  West Lafayette, IN 47906\\
  \texttt{mgaur@purdue.edu}
  \And
  Vaneet Aggarwal \\
  Purdue University\\
  West Lafayette, IN 47906\\
  \texttt{vaneet@purdue.edu}
}

\begin{document}
\maketitle

\input{abstract}

\input{intro}

\input{prelim}

\input{assumptions}

\input{main_result}

\input{stat_error}

\input{schedule}

\input{opti_error}

\input{experiments}

\input{conclusion}

\begin{ack}
This work was supported in part by the National Science Foundation (NSF) Future Manufacturing (FM) program under Award No. 2229250. The authors also thank Dhruv Sarkar and Aprameyo Chakrabarty for their valuable input and feedback.
\end{ack}

\newpage
\bibliography{ref}
\bibliographystyle{plain}

\input{appendix}

\newpage
\input{checklist.tex}
\end{document}

%% file: abstract.tex
\begin{abstract}
Flow-based generative models have emerged as an efficient alternative to diffusion models. We study the sample complexity of learning a target distribution in Wasserstein distance via flow matching, which learns a velocity field transporting a base distribution to the target. Under exact empirical risk minimization over a fixed network class and standard assumptions on the architecture and data distribution, we show that any squared-loss flow-matching variant including rectified flow with bounded derivatives achieves $\widetilde{O}(\varepsilon^{-2})$ sample complexity for the statistical term with the Wasserstein guarantee taking the form $W_2 \le K(\varepsilon + C\sqrt{\eapp})$, where $\eapp$ is the approximation error of the network class. This improves upon existing $\widetilde{O}(\varepsilon^{-4})$ guarantees and does not require the Polyak--\L{}ojasiewicz condition in the statistical analysis. Instead, we exploit a structural property of the squared-loss objective - it induces an approximate Bernstein condition that directly ties the variance of the excess loss to the excess risk up to additive approximation-error. This condition enables a localized Rademacher complexity analysis yielding fast $O(1/n)$ rates. The flow-matching schedule enters the leading statistical prefactor through the Lipschitz constant of the pointwise loss, which we characterize explicitly for standard schedules.
\end{abstract}

%% file: intro.tex
\section{Introduction}
\label{sec:intro}

Generative modeling has become a central tool in modern machine learning, driving advances across scientific discovery, vision, language, and biology \cite{shen2025efficient,wang2025geneflowtranslationsinglecellgene,yazdani2025generative}. A major milestone was the development of denoising diffusion probabilistic models (DDPMs) \cite{ho2020denoisingdiffusionprobabilisticmodels}. Score-based formulations followed, accelerating sampling via ODE representations \cite{song2021scorebasedgenerativemodelingstochastic}. These approaches remain computationally intensive. Flow-matching models \cite{lipman2023flowmatchinggenerativemodeling} offer a continuous-time alternative: they learn a velocity field that transports a base distribution to the target. Variants differ in their interpolation path and regression target, but all share a squared-loss training objective. Prominent instances include rectified flow \cite{liu2022flowstraightfastlearning}, which uses linear interpolation, and variance-preserving (VP) formulations with trigonometric or $\beta$-schedules.

We study the sample complexity of distribution learning via flow-matching models. Existing analyses establish $\widetilde{O}(\varepsilon^{-4})$ rates \cite{gaur2025improvedsamplecomplexitydiffusion,gaur2025generativemodelingcontinuousflows}. An $\Omega(\varepsilon^{-2})$ lower bound holds for integrated velocity estimation (\appref{app:lower_bound}). Many prior analyses that study this problem also rely on the Polyak--\L{}ojasiewicz (PL) condition, which is a strong curvature assumption on the loss landscape, and its validity for neural networks is difficult to verify in general. 
This motivates our central question:
\begin{quote}
Can flow-matching generative models achieve $\widetilde{O}(\varepsilon^{-2})$ sample complexity for the statistical term, without relying on curvature assumptions such as PL?
\end{quote}

We answer in the affirmative. The key observation is structural: the bias--variance decomposition of the squared loss yields an \emph{approximate Bernstein condition}. For any $\theta$ in the hypothesis class,
\begin{equation}\label{eq:intro-ab}
\E\|v_\theta - v_{\theta^*}\|^2
\le 2(\cL(\theta) - \cL(\theta^*)) + 4\eapp,
\end{equation}
where $v_{\theta^*}$ is the population minimizer in the parametric class $\{v_\theta : \theta \in \Theta\}$. The term $\eapp = \E\|v_{\theta^*} - v^*\|^2$ is the approximation error: the gap between the best in-class velocity field and the true regression function $v^*$. The condition holds for any squared-loss regression, so the analysis covers all flow-matching variants uniformly.

\vspace{-1mm}
\paragraph{Local Rademacher under approximate Bernstein.}
Equation~\eqref{eq:intro-ab} is the engine for our main technical result. The standard Bartlett--Bousquet--Mendelson framework \cite{bartlett2005local} for local Rademacher complexity requires the \emph{exact} Bernstein condition. We extend it to the approximate setting. The residual $\rho = 4\Lloss^2\eapp$ is carried through Talagrand's inequality and the peeling argument directly, where $\Lloss$ is the Lipschitz constant of the pointwise loss. The resulting inequality acquires only an additive $O(\eapp)$ floor. This gives a statistical error of $\widetilde{O}(\Lloss^2 P/n) + O(\eapp)$, with $P$ the number of network parameters. The Benton et al.\ reduction \cite{benton2024errorboundsflowmatching} then yields the Wasserstein bound
\[
W_2(\hat\pi_1, \pi_1) \le K\bigl(O(\varepsilon) + O(\sqrt{\eapp})\bigr),
\]
which requires $n = \widetilde{O}(\Lloss^2 P/\varepsilon^2)$ samples. The approximation error $\eapp$ is set by the expressiveness of the network class. The leading statistical prefactor depends on the schedule through $\Lloss$, which controls the pointwise loss on the truncated domain.

In short, two structural facts drive the rate. First, the squared-loss bias--variance decomposition yields approximate Bernstein for free, with no curvature assumption. Second, localization carrying the residual yields a fast rate up to a model-misspecification floor. The schedule-dependent constant $\Lloss$ provides a quantitative explanation for why some schedules outperform others in practice.

\begin{table*}[!t]
\centering
\renewcommand{\arraystretch}{1.3}
\begin{tabular}{p{1.7cm} p{2.1cm} p{2.4cm} p{1.2cm} p{4cm}}
\hline
\textbf{Work} & \textbf{Algorithm} & \textbf{Sample Complexity in $\widetilde{O}(\cdot)$} & \textbf{PL} & \textbf{Key Assumptions}\\
\hline

\cite{chen2024overviewdiffusionmodelsapplications}
& DDPM
& $\varepsilon^{-O(d)}$
& No
& Supported on unit cube
\\

\cite{block2022generativemodelingdenoisingautoencoders}
& DDPM
& $\frac{R^{6d} + R^{2d}}{\varepsilon^{4}} \frac{(\Theta^{2} P)^{D}}{\sqrt{D}}$
& No
& Ball-supported class approximates score
\\
\cite{gupta2025improvedsamplecomplexitybounds}
& DDPM
& $\frac{d^{2} P D}{\varepsilon^{5}} \cdot \log \Theta$
& No
& $2^{nd}$ moment bounded; ERM
\\

\cite{gaur2025improvedsamplecomplexitydiffusion}
& DDPM
& $\frac{W^{2D} \cdot d^{2}}{\varepsilon^{4}}$
& Yes
& Sub-Gaussian; no ERM
\\

\cite{gaur2025generativemodelingcontinuousflows}
& Flow matching
& $\frac{W^{2D-2} \, d^{2}}{\varepsilon^{4}}$
& Yes
& Bounded; no ERM
\\

\textbf{Our Work}
& \textbf{FM variant}
& $\frac{\Lloss^2 P}{\varepsilon^{2}}$
& No
& Sub-Gaussian; ERM
\\
\hline
\end{tabular}
\caption{Comparison of sample complexity results for distribution learning. The guarantees use different metrics and optimization assumptions.}
\label{tab:Table1}
\vspace{-.2in}
\end{table*}

Complementary end-to-end analyses of flow matching \cite{fukumizu2025flow} and diffusion models \cite{oko2023diffusion,fu2024conditional} control approximation and estimation under regularity assumptions. We focus on the statistical error, since this is where the rate improvement arises. The approximation error is controlled by network expressiveness. 
Our contributions are:
\begin{itemize}[leftmargin=*, itemsep=0pt]
    \item Any squared-loss flow-matching variant achieves statistical error $\widetilde{O}(\Lloss^2 P/n) + O(\eapp)$. This yields a Wasserstein guarantee $W_2 \le K(\varepsilon + C\sqrt{\eapp})$ with $n = \widetilde{O}(\Lloss^2 P/\varepsilon^2)$ samples. The result improves the $\varepsilon$-scaling from $\widetilde{O}(\varepsilon^{-4})$ \cite{gaur2025improvedsamplecomplexitydiffusion,gaur2025generativemodelingcontinuousflows} and makes model misspecification explicit.

    \item We use the squared-loss \emph{approximate Bernstein condition} (\cref{lem:approx-bernstein}) in place of the PL-based exact Bernstein condition in \cite{gaur2025generativemodelingcontinuousflows}; the residual-aware localization is developed in \cref{thm:approx-local-rad}.

    \item We develop a local Rademacher theory under approximate Bernstein (\cref{thm:approx-local-rad}). The Bartlett--Bousquet--Mendelson framework \cite{bartlett2005local} is extended by carrying the additive residual $\rho$ through Talagrand's inequality and the peeling argument. Only an $O(\eapp)$ floor is incurred in the final inequality. This broadens the scope of local Rademacher results.

    \item We characterize the schedule-dependent constant $\Lloss$ explicitly (\cref{sec:schedules}), giving a quantitative explanation for performance differences across schedules.

    \item We complement the upper bounds with a lower bound having the same $\varepsilon$ exponent (Appendix \ref{app:lower_bound}): $\Omega(\varepsilon^{-2})$ samples are necessary for expected integrated squared velocity error at most $\varepsilon^2$.
\end{itemize}

%% file: prelim.tex
\vspace{-4pt}
\section{Preliminaries}

\begin{wrapfigure}{r}{0.5\textwidth}

\footnotesize
\vspace{-8mm}
\begin{minipage}{0.5\textwidth}
\begin{algorithm}[H]
\caption{Flow Matching}
\label{alg:fm}
\begin{algorithmic}[1]
\REQUIRE Samples from $\pi_0$ and $\pi_1$, interpolation $I_t$, target $Z_t$, network $v_\theta$, iterations $T$
\STATE Draw $n$ i.i.d.\ pairs $(X_0^i, X_1^i)$ with $X_0^i \sim \pi_0$, $X_1^i \sim \pi_1$
\FOR{each pair $(X_0^i, X_1^i)$}
    \STATE Sample $t_i \sim \mathrm{Uniform}[0,1]$
    \STATE Set $X_{t_i}^i = I_{t_i}(X_0^i, X_1^i)$
    \STATE Set $Z_{t_i}^i = Z_{t_i}(X_0^i, X_1^i)$
\ENDFOR
\FOR{iteration $j = 1$ to $T$}
    \STATE Update $\theta$ to minimize:
    \STATE $\hat{\cL}(\theta) = \frac{1}{n} \sum_{i=1}^n \| v_\theta(X_{t_i}^i, t_i) - Z_{t_i}^i \|^2$
\ENDFOR
\STATE Sample $Z_0 \sim \pi_0$ and solve $\frac{dZ_t}{dt} = v_\theta(Z_t, t)$
\RETURN $Z_1$
\end{algorithmic}
\end{algorithm}
\end{minipage}

\end{wrapfigure}

In this section, we introduce the mathematical framework underlying flow-matching generative models, describe the general formulation, and present the associated learning objective.

Let $\pi_0$ and $\pi_1$ be probability distributions on $\R^d$. We assume that $\pi_0$ is a known and easily sampleable reference distribution (e.g., a standard Gaussian), while $\pi_1$ is observed only through i.i.d.\ samples. A flow model looks for a deterministic map pushing $\pi_0$ to $\pi_1$ through a continuous trajectory,
\vspace*{-2pt}
\begin{equation}
X : [0,1] \to \R^d, \qquad t \mapsto X_t,
\end{equation}
governed by a time-dependent velocity field $v : \R^d \times [0,1] \to \R^d$. The trajectory evolves according to the ordinary differential equation
\begin{equation}
dX_t = v(X_t, t)\,dt, \qquad X_0 = z.
\end{equation}
Sampling $X_0 \sim \pi_0$ and integrating this ODE yields a terminal variable $X_1$ whose distribution ideally matches $\pi_1$.

A \emph{flow-matching scheme} (see Algorithm~\ref{alg:fm}) specifies an interpolation map $I_t : \R^d \times \R^d \to \R^d$ and a target map $Z_t : \R^d \times \R^d \to \R^d$. Given paired samples $(X_0, X_1)$ from $\pi_0$ and $\pi_1$ For the distributional guarantees we consider the function \(t\mapsto I_t(x_0,x_1)\) is continuously differentiable and satisfies
\(
I_0(x_0,x_1)=x_0\) and \( I_1(x_0,x_1)=x_1.
\)
We denote interpolated point as
\(
X_t=I_t(X_0,X_1)\) and regression target as
\(Z_t=\partial_t I_t(X_0,X_1).
\) The population-optimal velocity field is

\begin{equation}
v^*(x,t) = \E[\,Z_t \mid X_t = x,\, t\,].
\end{equation}
In practice, $v^*$ is approximated by a neural network $v_\theta$, trained by minimizing the population loss
\begin{equation}\label{eq:pop-loss}
\cL(\theta) = \E_{t,X_0,X_1}\!\left[\| v_\theta(X_t, t) - Z_t \|^2\right],
\end{equation}
where $t \sim \mathrm{Uniform}[0,1]$. Its empirical counterpart is
\begin{equation}\label{eq:emp-loss}
\hat\cL(\theta) = \frac{1}{n} \sum_{i=1}^n \| v_\theta(X_t^i, t_i) - Z_t^i \|^2,
\end{equation}
with $(X_0^i, X_1^i, t_i)$ drawn i.i.d.\ from $\pi_0 \otimes \pi_1 \otimes \mathrm{Uniform}[0,1]$.

\begin{example}Standard flow-matching schemes
\label{ex:schemes}
\hfill
\begin{enumerate}[leftmargin=*,itemsep=0 pt,label=(\alph*)]
\item \textbf{Rectified flow / OT path} \cite{liu2022flowstraightfastlearning}.
$X_t = (1-t)X_0 + tX_1$, \quad $Z_t = X_1 - X_0$.

\item \textbf{VP (sin/cos).}
$X_t = \cos(\tfrac\pi2 t) X_0 + \sin(\tfrac\pi2 t) X_1$,
\quad $Z_t = \tfrac\pi2(-\sin(\tfrac\pi2 t) X_0 + \cos(\tfrac\pi2 t) X_1)$.

\item \textbf{VP ($\beta$-schedule).}
$X_t = \alpha_t X_0 + \sigma_t X_1$ with $\alpha_t = e^{-\frac12\int_0^t\beta(s)ds}$,
$\sigma_t = \sqrt{1-\alpha_t^2}$, so that $X_0 \sim \pi_0$ at $t=0$ and $X_t \to X_1 \sim \pi_1$ as $t \to \infty$ (or $t = 1$ after rescaling). If $\beta$ is continuous with $\beta(0)> 0$, then $|\dot\sigma_t/\alpha_t|\to\infty$ as $t\to 0$.
\end{enumerate}
\end{example}

\begin{remark}
In rectified flow, the regression target $Z_t = X_1 - X_0$ is independent of $t$. In general flow matching, $Z_t$ depends on $t$ through the schedule derivatives. This distinction does not affect our statistical analysis, the squared-loss bias-variance decomposition holds regardless, but it does affect the Lipschitz constant of the pointwise loss (see \cref{sec:schedules}).
\end{remark}

%% file: assumptions.tex
\vspace{-2pt}
\section{Assumptions}
\label{sec:assum}

To establish our sample-complexity results, we introduce assumptions on the data distribution, the flow-matching schedule, the neural-network architecture, and the approximation quality. These conditions are standard in theoretical analyses of generative models. Notably, we do \textbf{not} assume the Polyak--\L{}ojasiewicz condition.

\begin{assumption}[\textbf{Data distributions}]
\label{ass:data}
Samples $(X_0^i, X_1^i, t_i)$ are i.i.d.\ draws from $\pi_0 \otimes \pi_1 \otimes \mathrm{Uniform}[0,1]$. The random variables $X_0 \sim \pi_0$ and $X_1 \sim \pi_1$ satisfy sub-Gaussian tail bounds: there exist constants $c, \sigma > 0$ such that for all $u \in \R^d$ with $\|u\|_2 = 1$,
\begin{equation}
\Pr\big(|\langle u, X_j \rangle| \ge t\big)
\le 2\exp\!\left(-\frac{t^2}{c\sigma^2}\right), \qquad j \in \{0,1\}.
\end{equation}
\end{assumption}

\begin{remark}
The i.i.d.\ assumption enables concentration of empirical averages, which is required for uniform convergence. Sub-Gaussianity rules out heavy tails and provides the concentration properties needed for our analysis. These conditions are standard in recent theoretical work on diffusion models \cite{gaur2025improvedsamplecomplexitydiffusion,gupta2025improvedsamplecomplexitybounds}.
\end{remark}

\begin{assumption}[\textbf{Schedule regularity}]
\label{ass:schedule}
There exists a constant $c_Z > 0$ such that for all $t \in [0,1]$,
\begin{equation}
\|Z_t(x_0, x_1)\| \le c_Z(\|x_0\| + \|x_1\| + 1).
\end{equation}
The interpolation also satisfies $\|I_t(x_0,x_1)\|\le c_I(\|x_0\|+\|x_1\|+1)$ for a constant $c_I>0$.
\end{assumption}

\begin{remark}
This assumption controls the growth of the regression target and interpolation. It holds with $c_Z = 1$ for rectified flow, $c_Z = \pi/2$ for VP sin/cos, and fails for VP $\beta$-schedule with $\beta(0)>0$ without endpoint clipping. All listed interpolations have $c_I=1$. The constant $c_Z$ enters the bound through $\Lloss$; see \cref{sec:schedules} for a detailed discussion.
\end{remark}

\begin{assumption}[\textbf{Neural-network architecture}]
\label{ass:nn}
Let $\cF$ denote the class of $\R^d$-valued functions represented by feed-forward neural networks with the following properties:
\begin{itemize}[leftmargin=*,itemsep=0pt, topsep=1pt, partopsep=0pt, parsep=0pt]
\item It has $D \ge 2$ affine layers, a linear output layer, and connections only between successive layers.
\item It contains a total of $P$ trainable parameters.
\item Each hidden layer output is bounded in $[-b,b]$.
\item The activation functions in the first hidden layer are coordinate-wise non-decreasing.
\item For each layer, the corresponding weight matrix $W$ satisfies $\|W\|_{\mathrm{op}} \le V$ (spectral/operator norm); biases, when present, are included in this constrained matrix by augmenting the input with a constant coordinate. Hidden-layer activations are differentiable and $L_\phi$-Lipschitz.
\end{itemize}
\end{assumption}

\begin{remark}
These conditions are standard in neural network theory \cite{anthony1999neural}. Common activations satisfying these conditions include $\tanh$ and $\mathrm{sigmoid}$ whose ranges are bounded. The Bernstein argument is activation-independent; extending the result to other activations requires suitable output and covering bounds.
\end{remark}

\begin{assumption}[\textbf{Existence and Uniqueness}]
\label{ass:existence}
For each $x \in \R^d$ and $s \in [0,1]$, the flow ODEs with velocity fields $v_\theta(x,t)$ and $v^*(x,t)$ have unique solutions that are continuously differentiable in $x$, $s$, and $t$.
\end{assumption}

\begin{remark}
This ensures that the flow ODEs are well-posed. For $v_\theta$, existence and uniqueness follow from the Lipschitz continuity in $x$ (Corollary 3) via Picard–Lindelöf; differentiability of the flow maps, and well-posedness for $v^*$, are assumed directly. This assumption also appears in \cite{benton2024errorboundsflowmatching} and \cite{gaur2025generativemodelingcontinuousflows}.
\end{remark}

\begin{assumption}[\textbf{Approximation error}]
\label{ass:approx}
Let $\theta^* := \arg\min_{\theta \in \Theta} \cL(\theta)$. The approximation error is
\[
\eapp := \E_{t,X_0,X_1}\|v_{\theta^*}(X_t,t) - v^*(X_t,t)\|^2.
\]
\end{assumption}

\begin{remark}
We do \emph{not} assume $\eapp$ is small. It appears as an explicit additive floor in all our bounds including the final Wasserstein guarantee making the role of model misspecification explicit. The bound separates this approximation floor from the statistical error.
\end{remark}

\begin{remark}[No PL assumption]
Unlike prior work \cite{gaur2025improvedsamplecomplexitydiffusion,gaur2025generativemodelingcontinuousflows}, we do not assume the Polyak--\L{}ojasiewicz condition. The Bernstein-type variance control needed for fast statistical rates is derived directly from the squared-loss structure (\cref{lem:approx-bernstein}), not from curvature of the loss landscape.
\end{remark}

\begin{assumption}[\textbf{Access to ERM}]
\label{ass:erm}
We assume access to an empirical risk minimizer
\[
\hat\theta \in \arg\min_{\theta \in \Theta} \frac{1}{n} \sum_{i=1}^n \|v_\theta(X_t^i, t_i) - Z_t^i\|^2.
\]
That is, the parameter $\theta$ produced by the training procedure satisfies $\theta = \hat\theta$, so the optimization error $\mathcal{E}^{\mathrm{opt}} = \E\|v_{\hat\theta} - v_\theta\|^2 = 0$.
\end{assumption}

\begin{remark}[Role of \cref{ass:erm}]
\label{rem:erm-role}
The optimization error $\mathcal{E}^{\mathrm{opt}}$ measures the gap between the empirical risk minimizer (ERM) $\hat\theta$ and the parameter $\theta$ produced by the training procedure. Under \cref{ass:erm}, $\theta = \hat\theta$, so $\mathcal{E}^{\mathrm{opt}} = 0$ and the statistical rate in \cref{thm:stat-main} fully characterizes the estimation error. This assumption isolates the statistical contribution of distribution learning from optimization effects and is standard in prior sample-complexity analyses of diffusion and flow-matching models \cite{block2022generativemodelingdenoisingautoencoders,gupta2025improvedsamplecomplexitybounds}. Without \cref{ass:erm}, an optimization guarantee must be established separately under conditions applicable to the chosen architecture, data, and training procedure. Appendix~\ref{app:opt} discusses how a verified approximate-ERM guarantee can be incorporated without changing the statistical order.
\end{remark}

The following properties are consequences of \cref{ass:nn} on the truncated event $\mathcal{E}_M := \{\|X_0\| \le M,\, \|X_1\| \le M\}$ defined in \cref{sec:subgaussian}. We mention the results here and give detailed proofs in Appendix \ref{app:ass_proved}.
\vspace{-4mm}
\begin{itemize}[itemsep=-0.4mm]
\item The parameter space $\Theta$ is compact.
\item The network output $v_\theta(x,t)$ is uniformly bounded by some $B_v > 0$ over $\theta \in \Theta$ and $(x,t)$.
\item For every $\theta \in \Theta$, the function $v_\theta(x,t)$ is differentiable in $(x,t)$ and $L_t$-Lipschitz in the state variable $x$.
\item Let $\loss$ denote the pointwise loss $\loss(v_\theta, Z_t) = \|v_\theta - Z_t\|_2^2$. Then, on $\mathcal{E}_M$, $\loss$ is $\Lloss$-Lipschitz in its first argument, with $\Lloss = 2(B_v + c_Z(2M+1))$.
\end{itemize}

%% file: main_result.tex
\vspace{-4pt}
\section{Main Result}

In this section, we present our main theoretical results on the sample complexity of flow matching. We begin by recalling a prior result in \cite{benton2024errorboundsflowmatching} that links the estimation error of the learned velocity field to the Wasserstein distance between the generated and target distributions. Building on this connection, we introduce an error decomposition that isolates approximation, statistical, and optimization effects.

\subsection{Prior Result}

\begin{theorem}[{\cite{benton2024errorboundsflowmatching}}]
\label{thm:benton}
Let $\pi_0, \pi_1$ be probability distributions on $\R^d$. Let $v^*$ be the true velocity field and $v_\theta$ the learned field. Define $Y_t$ and $Z_t$ as the flows generated by $v_\theta$ and $v^*$ respectively, both starting from $Y_0 = Z_0 \sim \pi_0$. Let $\hat\pi_1$ be the law of $Y_1$ and $\pi_1$ the law of $Z_1$.

Under the following assumptions:
\begin{enumerate}[leftmargin=*,itemsep=0pt, topsep=2pt, partopsep=0pt, parsep=0pt]
\item \textbf{Bounded Velocity Error:}
$\int_0^1 \E[\|v_\theta(X_t, t) - v^*(X_t, t)\|^2] \, dt \le \varepsilon^2$.
\item \textbf{Regularity:} $v_\theta(x,t)$ is differentiable in both inputs and $L_t$-Lipschitz in $x$.
\item \textbf{Existence and Uniqueness:} Same as \cref{ass:existence}.
\end{enumerate}

The Wasserstein-2 distance satisfies:
\begin{equation}\label{eq:wasser}
W_2(\hat\pi_1, \pi_1) \le \varepsilon \cdot \exp\!\left(\int_0^1 L_t \, dt\right).
\end{equation}
\end{theorem}

Assumptions 2 and 3 in \cref{thm:benton} are consequences of \cref{ass:nn} and \cref{ass:existence}, so a small $L^2$ velocity error directly yields small Wasserstein discrepancy.

\vspace{-2pt}
\subsection{Error Decomposition}
\label{sec:error-decomp}

Thanks to \cref{thm:benton}, controlling the velocity estimation error becomes our main objective. To analyze this quantity, we introduce the population and empirical risk minimizers:
\vspace{-1mm}
\begin{equation}
\theta^* = \arg\min_{\theta \in \Theta} \E_{t,X_0,X_1}\!\left[\|v_\theta(X_t,t) - Z_t\|^2\right],
\end{equation}
\begin{equation}
\hat\theta = \arg\min_{\theta \in \Theta} \frac{1}{n} \sum_{i=1}^n \|v_\theta(X_t^i,t_i) - Z_t^i\|^2.
\end{equation}

Let $\theta$ denote the parameter actually produced by the optimizer.

\begin{proposition}[Error Decomposition]
\label{prop:err-dec}
The velocity field estimation error decomposes as:
\begin{align*}
\E\|v_{\theta}(X_t,t) - v^*(X_t,t)\|^2
&\le 3\underbrace{\E\|v^* - v_{\theta^*}\|^2}_{=\,\eapp\ \text{(approximation)}}
+ 3\underbrace{\E\|v_{\theta^*} - v_{\hat\theta}\|^2}_{\mathcal{E}^{\mathrm{stat}}\ \text{(statistical)}}
+ 3\underbrace{\E\|v_{\hat\theta} - v_{\theta}\|^2}_{\mathcal{E}^{\mathrm{opt}}\ \text{(optimization)}}.
\end{align*}
Apply $\|a+b+c\|^2 \le 3(\|a\|^2+\|b\|^2+\|c\|^2)$ to $v^* - v_{\theta} = (v^* - v_{\theta^*}) + (v_{\theta^*} - v_{\hat\theta}) + (v_{\hat\theta} - v_{\theta})$.
\end{proposition}

\subsection{Our Main Result}

\begin{theorem}[Sample Complexity of Flow Matching]
\label{thm:main}
Let the assumptions stated in \cref{sec:assum} hold. Under \cref{ass:erm}, the optimization error vanishes ($\mathcal{E}^{\mathrm{opt}} = 0$); \appref{app:opt} discusses approximate ERM under separately verified optimization assumptions. A sample size $n = \widetilde{O}(\Lloss^2 P / \varepsilon^2)$ suffices to ensure, with probability at least $1 - \delta$, that the velocity estimation error satisfies
\begin{equation}\label{eq:vel-error-main}
\int_0^1 \E\|v_{\theta}(X_t,t) - v^*(X_t,t)\|^2 \, dt
\le C_1\varepsilon^2 + C_2\eapp,
\end{equation}
where $C_1, C_2 > 0$ are universal constants. Consequently, by \cref{thm:benton},
\begin{equation}\label{eq:W2-main}
W_2(\hat\pi_1, \pi_1)
\le K\sqrt{C_1\varepsilon^2 + C_2\eapp}
\le K\bigl(\sqrt{C_1}\,\varepsilon + \sqrt{C_2\eapp}\bigr),
\end{equation}
where $K = \exp(\int_0^1 L_t \, dt)$.
\end{theorem}

\begin{proof}[Proof Sketch]
By \cref{prop:err-dec}:
\[
\E\|v_{\theta} - v^*\|^2
\le 3\eapp + 3\mathcal{E}^{\mathrm{stat}} + 3\mathcal{E}^{\mathrm{opt}}.
\]

 By \cref{thm:stat-main} (proved in \cref{sec:stat-error}), with probability at least $1-\delta$ (taking $x=\log(2/\delta)$),
\[
\mathcal{E}^{\mathrm{stat}}
\le \widetilde{O}\!\left(\frac{\Lloss^2 P}{n}\right) + 8\eapp.
\]
Choosing a sufficiently large $n = \widetilde{O}(\Lloss^2 P/\varepsilon^2)$ gives $\mathcal{E}^{\mathrm{stat}} \le \varepsilon^2 + 8\eapp$.

Substituting $\mathcal{E}^{\mathrm{opt}} = 0$ (\cref{ass:erm}), we have
\[
\E\|v_{\theta} - v^*\|^2
\le 3\eapp + 3(\varepsilon^2 + 8\eapp)
= 3\varepsilon^2 + 27\eapp,
\]
which gives~\eqref{eq:vel-error-main} with $C_1 = 3$ and $C_2 = 27$. The Wasserstein bound~\eqref{eq:W2-main} follows from \cref{thm:benton} and $\sqrt{a+b} \le \sqrt{a} + \sqrt{b}$.
\end{proof}

\begin{remark}
\cref{thm:main} cleanly separates two independent requirements for accurate distribution learning. The \emph{sample complexity} $n = \widetilde{O}(\Lloss^2 P/\varepsilon^2)$ controls the statistical error for the fixed class. The \emph{approximation error} $\eapp$ is an architecture-dependent floor controlled by the expressiveness of the network class: a richer architecture decreases $\eapp$ at the cost of increasing $P$.
\end{remark}

%% file: stat_error.tex
\section{Statistical Error}
\label{sec:stat-error}

This section presents the core technical contribution of the paper: a localized Rademacher complexity analysis for squared-loss flow matching that achieves fast $\widetilde{O}(1/n)$ rates without the PL condition. The argument proceeds in three steps. First, we derive an approximate Bernstein condition from the squared-loss structure. Second, we feed it into a localization argument that directly extends the standard local Rademacher framework, carrying the additive approximation residual through the proof. Third, we combine the resulting  inequality with a covering bound for the network class. The most non-trivial steps are in proving the approximate Bernstein version of local Rademacher results are inspired from \cite{bartlett2005local}.

Since $v^*(x,t) = \E[Z_t \mid X_t = x, t]$ is the regression function, the bias--variance decomposition gives, for any measurable $f$,
\begin{equation}\label{eq:bias-var}
\E\|f(X_t,t) - Z_t\|^2 = \E\|f(X_t,t) - v^*(X_t,t)\|^2 + \E\|Z_t - v^*(X_t,t)\|^2.
\end{equation}
The second term does not depend on $f$. The excess risk therefore equals the difference in squared population errors:
\begin{equation}\label{eq:excess-identity}
\cL(\theta) - \cL(\theta^*) = \E\|v_\theta - v^*\|^2 - \E\|v_{\theta^*} - v^*\|^2.
\end{equation}

\subsection{Approximate Bernstein condition}
\label{sec:approx-bernstein}

The first step is a Bernstein-type bound that relates $\E\|v_\theta - v_{\theta^*}\|^2$ to the excess risk with additional approximation error term.
\begin{lemma}
\label{lem:approx-bernstein}
For every $\theta \in \Theta$, writing $\Delta_\theta=\cL(\theta)-\cL(\theta^*)$,
\begin{equation}\label{eq:approx-bernstein}
\E\|v_\theta(X_t,t) - v_{\theta^*}(X_t,t)\|^2
\le\left(\sqrt{\eapp}+\sqrt{\eapp+\Delta_\theta}\right)^2
\le 2\Delta_\theta+4\eapp.
\end{equation}
\end{lemma}
\begin{proof}
Expand $v_\theta - v^* = (v_\theta - v_{\theta^*}) + (v_{\theta^*} - v^*)$ in~\eqref{eq:excess-identity} to get
\begin{equation}\label{eq:expand}
\cL(\theta) - \cL(\theta^*) = \E\|v_\theta - v_{\theta^*}\|^2 + 2\E\langle v_\theta - v_{\theta^*},\; v_{\theta^*} - v^*\rangle,
\end{equation}
and rearrange:
\begin{equation}\label{eq:rearrange}
\E\|v_\theta - v_{\theta^*}\|^2 = \cL(\theta) - \cL(\theta^*) - 2\E\langle v_\theta - v_{\theta^*},\; v_{\theta^*} - v^*\rangle.
\end{equation}
By Cauchy--Schwarz,
\begin{equation}\label{eq:cs}
\bigl|2\E\langle v_\theta - v_{\theta^*},\; v_{\theta^*} - v^*\rangle\bigr|
\le 2\bigl(\E\|v_\theta - v_{\theta^*}\|^2\bigr)^{1/2}\sqrt{\eapp},
\end{equation}
and substitution into~\eqref{eq:rearrange} yields
\begin{equation}\label{eq:before-young}
\E\|v_\theta - v_{\theta^*}\|^2
\le \cL(\theta) - \cL(\theta^*) + 2\bigl(\E\|v_\theta - v_{\theta^*}\|^2\bigr)^{1/2}\sqrt{\eapp}.
\end{equation}
Writing $\Delta_\theta=\cL(\theta)-\cL(\theta^*)$, the sharper form is
\[
\E\|v_\theta-v_{\theta^*}\|^2
\le \left(\sqrt{\eapp}+\sqrt{\eapp+\Delta_\theta}\right)^2.
\]
This follows by solving the quadratic inequality in \eqref{eq:before-young}; the relaxation below is convenient for the subsequent analysis.

Young's inequality $2ab \le \tfrac12 a^2 + 2 b^2$ with $a = (\E\|v_\theta - v_{\theta^*}\|^2)^{1/2}$ and $b = \sqrt{\eapp}$ then gives
\[
\E\|v_\theta - v_{\theta^*}\|^2
\le \cL(\theta) - \cL(\theta^*)
+ \tfrac12 \E\|v_\theta - v_{\theta^*}\|^2 + 2\eapp.
\]
Subtracting $\tfrac12 \E\|v_\theta - v_{\theta^*}\|^2$ from both sides and multiplying by $2$ yields~\eqref{eq:approx-bernstein}.
\end{proof}

To pass from a function-level bound to a variance bound at the loss level, define the pointwise excess loss
\begin{equation}\label{eq:g-def}
g_\theta := \|v_\theta(X_t,t) - Z_t\|^2 - \|v_{\theta^*}(X_t,t) - Z_t\|^2,
\end{equation}
so that $\Delta_\theta:=\E[g_\theta]=\cL(\theta)-\cL(\theta^*)\ge0$. For bounded targets, let $h_\theta=g_\theta$; for sub-Gaussian data, let $h_\theta=g_\theta\mathbf1_{\mathcal E_M}$, where $\mathcal E_M$ is defined in \cref{sec:subgaussian}. If $\|Z_t\|\le M'$ on the domain where $h_\theta$ is nonzero, set $\Lloss=2(B_v+M')$ and $B_\ell=(B_v+M')^2$. Then $|h_\theta|\le B_\ell$, and the identity $\|a\|^2-\|b\|^2=\langle a-b,a+b\rangle$ gives
\begin{equation}\label{eq:g-lip}
|g_\theta|\le\|v_\theta-v_{\theta^*}\|\,(2B_v+2\|Z_t\|)\le\Lloss\|v_\theta-v_{\theta^*}\|
\end{equation}
on that domain. Under the original sampling distribution,
\begin{equation}\label{eq:var-control}
\Var(h_\theta)\le Ph_\theta^2\le\Lloss^2\E\|v_\theta-v_{\theta^*}\|^2.
\end{equation}
Inserting \cref{lem:approx-bernstein} produces the loss-level Bernstein bound
\begin{equation}\label{eq:ab-variance}
Ph_\theta^2\le\underbrace{2\Lloss^2}_{=:\,B_0}\Delta_\theta+\underbrace{4\Lloss^2\eapp}_{=:\,\rho}.
\end{equation}

\subsection{Localized Rademacher complexity with approximate Bernstein}
\label{sec:local-rad}

The standard localized Rademacher framework \cite{bartlett2005local,koltchinskii2006local} assumes an exact Bernstein condition. We extend it to handle \eqref{eq:ab-variance}, carrying the residual $\rho$ through concentration and peeling. We localize the loss class itself by $Ph_\theta^2$. The following form also retains the small population discrepancy introduced by truncation. Write
$R_n(\mathcal H)=\E_\epsilon\sup_{h\in\mathcal H}|n^{-1}\sum_i\epsilon_i h(\xi_i)|$.

\begin{theorem}
\label{thm:approx-local-rad}
Let $\mathcal H=\{h_\theta:\theta\in\Theta\}$ contain zero, with $|h_\theta|\le B_\ell$, $B_0,B_\ell>0$, and $\rho,\tau\ge0$. Suppose $\Delta_\theta\ge0$, $|Ph_\theta-\Delta_\theta|\le\tau$, and $Ph_\theta^2\le B_0\Delta_\theta+\rho$. Let $\psi$ be a sub-root function with positive fixed point $r^*$ satisfying
\[
\psi(r)\ge B_0\E R_n\{h\in\mathcal H:Ph^2\le r\}.
\]
Set $J=\lceil\log_2\max\{1,B_\ell^2/r^*\}\rceil$ and $y=x+\log(J+1)$. For $K>1$ and $x>0$, with probability at least $1-e^{-x}$, every $\hat\theta$ satisfying $P_nh_{\hat\theta}\le0$ obeys
\begin{equation}\label{eq:approx-erm}
\Delta_{\hat\theta}\le
\frac{C_1K^2(B_0+B_\ell)r^*}{(K-1)B_0^2}
+\frac{C_2K^2(B_0+B_\ell)y}{(K-1)n}
+\frac K{K-1}\left(\frac\rho{2B_0}+\tau\right),
\end{equation}
with $C_1=48$ and $C_2=9$.
\end{theorem}

The proof uses finitely many second-moment shells and retains their confidence cost $\log(J+1)$; see \appref{app:sec:shift}. For bounded targets, $h=g$ and $\tau=0$. With $K=2$, $B_0=2\Lloss^2$, and $\rho=4\Lloss^2\eapp$, the approximation term is $\rho/B_0=2\eapp$.

It remains to control $r^*$, which requires a covering bound for the network class.
\begin{lemma}[Covering-number bound]
\label{lem:covering}
Under \cref{ass:nn}, restrict the vector-valued network class to $\|x\|\le R$, $t\in[0,1]$. If $L_R$ is its parameter Lipschitz constant from \cref{corr:lipschitz_theta}, then
\[
\log\mathcal N_\infty(u,\mathcal F)\le P\log\left(1+\frac{2V\sqrt P L_R}{u}\right),\qquad u>0,
\]
where the function distance uses the Euclidean norm of the vector output.
\end{lemma}

For the corresponding loss class, define
\[
A_R=B_\ell+4V\sqrt P\,\Lloss L_R,\qquad
H_R=1+P\log\left(1+\frac{A_Rn}{B_\ell}\right),\qquad q_R=\frac{H_R}{n}.
\]
An empirical Dudley bound with deterministic cutoff $B_\ell/n$, combined with symmetrization and contraction for the squared loss class, gives
\[
\E R_n\{h\in\mathcal H:Ph^2\le r\}\le C\bigl(\sqrt{rq_R}+B_\ell q_R\bigr).
\]
This step explicitly transfers the population radius to the empirical radius (\appref{app:sec:dudley}). A valid sub-root function is therefore
\begin{equation}\label{eq:subroot-main}
\psi(r)=CB_0\bigl(\sqrt{rq_R}+B_\ell q_R\bigr)+\frac{B_0^2}{n}.
\end{equation}
Its fixed point satisfies
\begin{equation}\label{eq:rstar}
\frac{B_0^2}{n}\le r^*\le C'(B_0^2+B_0B_\ell)\frac{H_R}{n}.
\end{equation}
The lower bound follows directly from the last term of \eqref{eq:subroot-main}. Since $B_\ell=B_0/8$, the shell contribution $\log(J+1)$ is absorbed by $H_R\ge1+\log(1+n)$.

\label{sec:stat-bounded}

Combining \cref{thm:approx-local-rad} with \eqref{eq:rstar} and choosing $K=2$ gives the bounded-loss bound.
\begin{lemma}
\label{lem:stat-bounded}
Suppose $\|X_t\|\le R$ and $\|Z_t\|\le M'$ almost surely. Under \cref{ass:nn,ass:approx,ass:erm}, for every $x>0$, with probability at least $1-e^{-x}$,
\begin{equation}\label{eq:excess-bounded}
\cL(\hat\theta)-\cL(\theta^*)\le\frac{C\Lloss^2(H_R+x)}{n}+2\eapp,
\end{equation}
where $C$ is universal. Converting to velocity error via \cref{lem:approx-bernstein} yields the dominant term $\widetilde O(\Lloss^2P/n)$ and the floor $8\eapp$.
\end{lemma}

\paragraph{Extension to sub-Gaussian data.}
\label{sec:subgaussian}
Define $\mathcal E_M=\{\|X_0\|\le M,\|X_1\|\le M\}$ and choose
\[
M^2=4c\sigma^2\bigl(d\log5+x+\log(4n^2)\bigr),
\]
where $c,\sigma$ are from \cref{ass:data}. On $\mathcal E_M$, set $R_M=c_I(2M+1)$ and $M'=c_Z(2M+1)$. We apply the theorem to $h_\theta=g_\theta\mathbf1_{\mathcal E_M}$ under the original sampling law. The event that all $n$ observations lie in $\mathcal E_M$ has failure probability at most $e^{-x}/n$, and the population truncation discrepancy satisfies $\tau\le\Lloss^2e^{-x}/n^2$ (\appref{app:sec:subgauss}). On this event, the original ERM satisfies $P_nh_{\hat\theta}\le0$.

\begin{theorem}[Statistical error]
\label{thm:stat-main}
Under \cref{ass:data,ass:schedule,ass:nn,ass:approx,ass:erm}, for $n\ge2$ and every $x>0$, with probability at least $1-2e^{-x}$,
\begin{equation}\label{eq:stat-main}
\mathcal E^{\mathrm{stat}}:=\E\|v_{\hat\theta}-v_{\theta^*}\|^2
\le\frac{C\Lloss^2(H_{R_M}+x)}{n}+8\eapp,
\end{equation}
where $\Lloss=2(B_v+c_Z(2M+1))$, $H_{R_M}$ is defined above, and $C$ is universal. Here $M=O(\sigma\sqrt{c(d+\log n+x)})$ and $L_{R_M}$, hence $H_{R_M}$, displays the architecture and input-radius dependence explicitly.
\end{theorem}

%% file: schedule.tex
\section{Schedule-Dependent Constants}
\label{sec:schedules}

The statistical rate in \cref{thm:stat-main} is $\widetilde{O}(\Lloss^2 P/n) + O(\eapp)$. The schedule enters through $\Lloss = 2(B_v + c_Z(2M+1))$, where $c_Z$ is the schedule constant from \cref{ass:schedule} and $M = O(\sigma\sqrt{d+\log(n/\delta)})$ is the truncation level at confidence $1-\delta$.

\begin{proposition}[Schedule-dependent constants]
\label{prop:schedule}
Under \cref{ass:schedule}, $\Lloss = O(B_v+c_Z(1+\sigma\sqrt{d+\log(n/\delta)}))$. For fixed architecture and confidence, the following convenient upper bounds apply to standard schemes:
\begin{center}
\begin{tabular}{lcc}
\toprule
Schedule & $c_Z$ & $\Lloss$ scaling \\
\midrule
Rectified flow / OT & $1$ & $O(\sigma\sqrt{d\log n})$ \\
VP sin/cos & $\pi/2$ & $O(\sigma\sqrt{d\log n})$ \\
VP $\beta$-schedule (coefficient floor $\tau$) & $O(1/\tau)$ & $O(\sigma\sqrt{d\log n}/\tau)$ \\
\bottomrule
\end{tabular}
\end{center}
\end{proposition}

\begin{remark}
For a clipped VP $\beta$-schedule, let $\tau\in(0,1]$ denote the endpoint coefficient floor, rather than time. Suppose its endpoint bias is $O(\tau)$ and $c_Z\le C_Z/\tau$. Define
\[
A:=C\bigl[B_v^2+C_Z^2(2M+1)^2\bigr]
\bigl[H_{R_M}+\log(2/\delta)\bigr],
\]
where $H_{R_M}$ is the entropy factor used in \eqref{eq:rstar}, containing the parameter count $P$, depth, weight and activation bounds, and truncated input radius. Then \cref{thm:stat-main} bounds the statistical velocity error by $A/(n\tau^2)$, up to the approximation floor. The statistical and clipping contributions to the Wasserstein bound take the form
\[
K_\tau\!\left(\tau+\frac{\sqrt{A/n}}{\tau}\right),
\]
where $K_\tau=\exp(\int L_{t,\tau}\,dt)$ is the flow-stability factor. If the endpoint-bias constant, $A$, and $K_\tau$ are uniformly controlled up to logarithms as $\tau\downarrow0$, balancing gives $\tau\asymp(A/n)^{1/4}$ and a $\widetilde O(n^{-1/4})$ contribution. Under these conditions, $\widetilde O(\varepsilon^{-4})$ samples suffice for this contribution to be $O(\varepsilon)$; the approximation floor is added separately. The bias here is an endpoint error, not numerical discretization error.
\end{remark}


%% file: opti_error.tex
%% file: experiments.tex
\section{Numerical Experiments}
\vspace{-3mm}
\label{sec:experiments}

To examine the scaling predicted by Theorem~\ref{thm:main}, we train the same
$4$-layer width-$512$ MLP across four data distributions of increasing
dimension $d \in \{2, 8, 16, 32\}$: a 2D Gaussian and three axis-aligned
GMMs. The SiLU experiments illustrate behavior beyond the bounded-activation class used in our analysis. Figure~\ref{fig:rate-main} shows the excess velocity
error $L(\hat\theta_n) - L^\ast$ vs.\ training-set size $n$ on log-log
axes. Rectified flow attains slope close to the predicted $-1$ in every
panel, consistent with the predicted $\widetilde{O}(L_\ell^2 P / n)$ scaling.
The ordering RF~$<$~VP-SC~$<$~VP-$\beta$ widens with $d$, consistent
with the schedule-dependent constant
$L_\ell = O(c_Z \sigma \sqrt{d \log n})$ from
Section~\ref{sec:schedules} entering quadratically into the
statistical-error prefactor. VP-SC and VP-$\beta$ further exhibit mild
slope-flattening at higher $d$. A nonzero approximation floor may
contribute to this behavior; a larger statistical prefactor alone
does not change the predicted slope. Full setup and
evaluation methodology are deferred to Appendix~\ref{app:experiments}.

\begin{figure}[t]
    \centering
    \begin{subfigure}[b]{0.42\textwidth}
        \includegraphics[width=\textwidth]{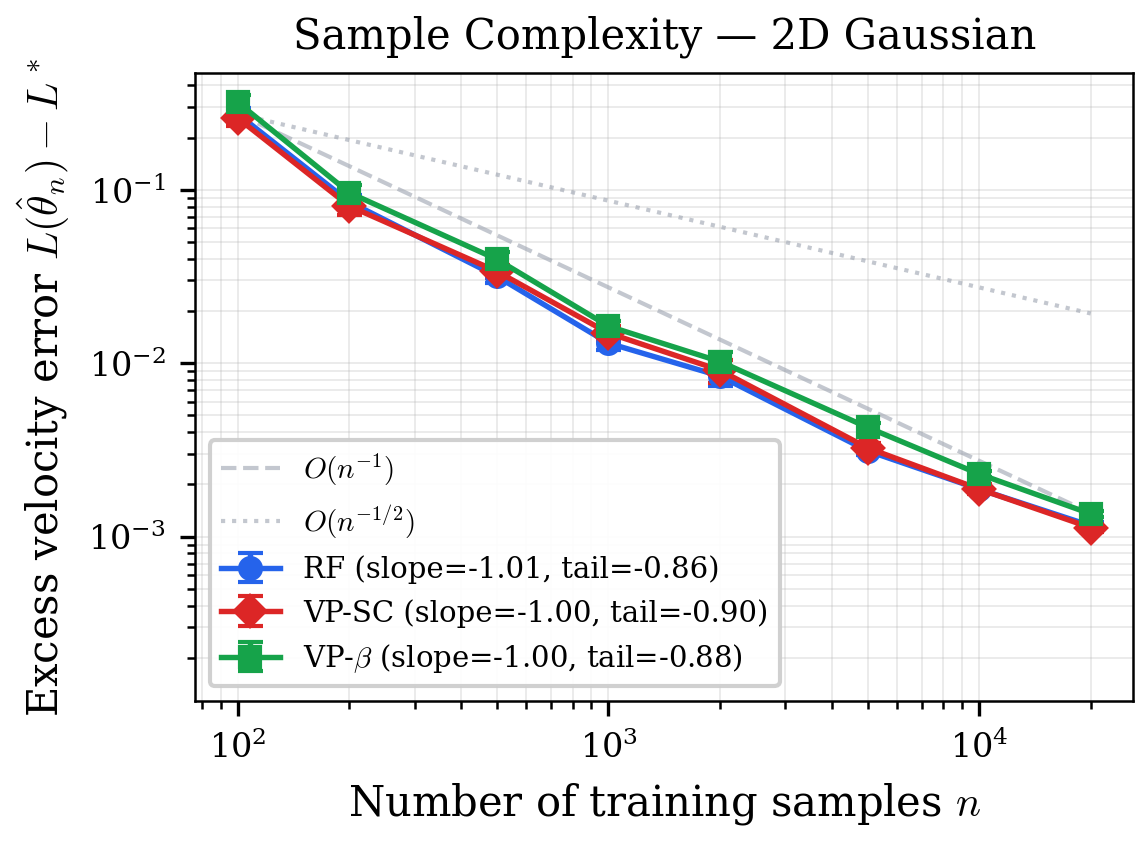}
        \caption{2D Gaussian}
    \end{subfigure}
    \begin{subfigure}[b]{0.42\textwidth}
        \includegraphics[width=\textwidth]{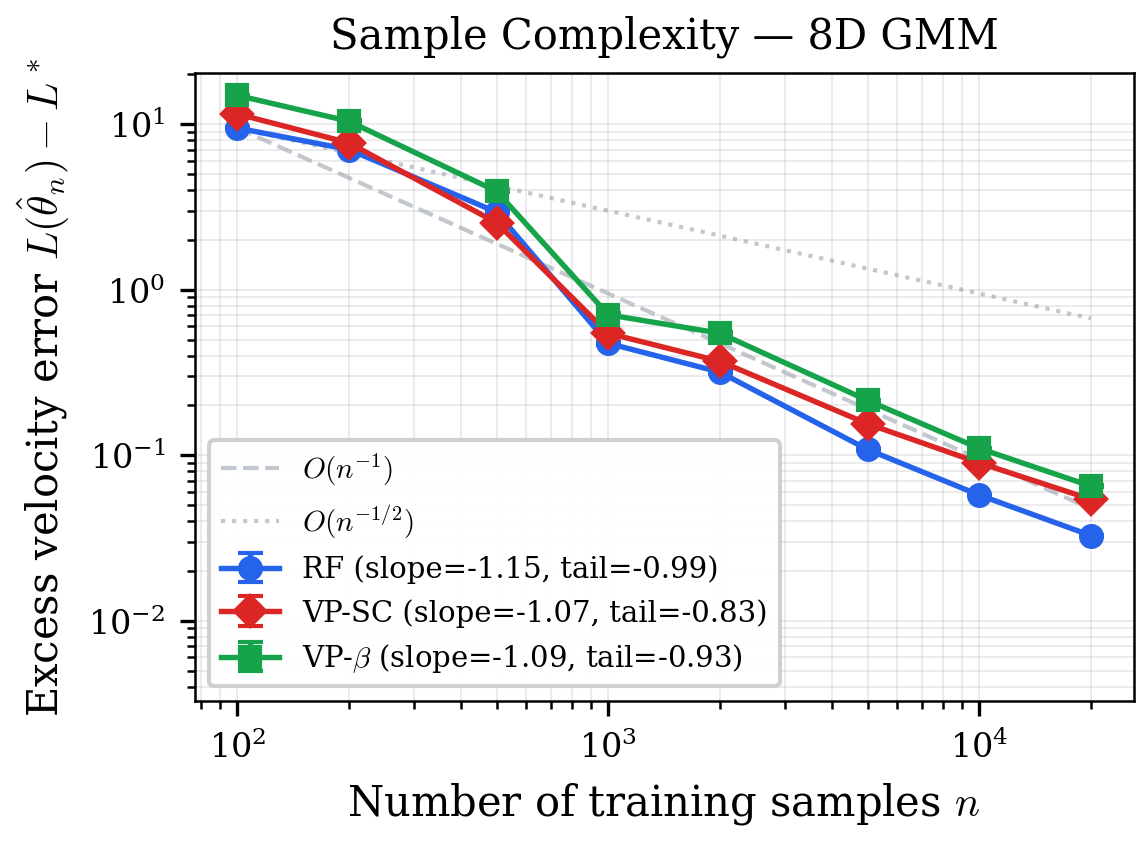}
        \caption{8D GMM}
    \end{subfigure}
    \\[0.5em]
    \begin{subfigure}[b]{0.42\textwidth}
        \includegraphics[width=\textwidth]{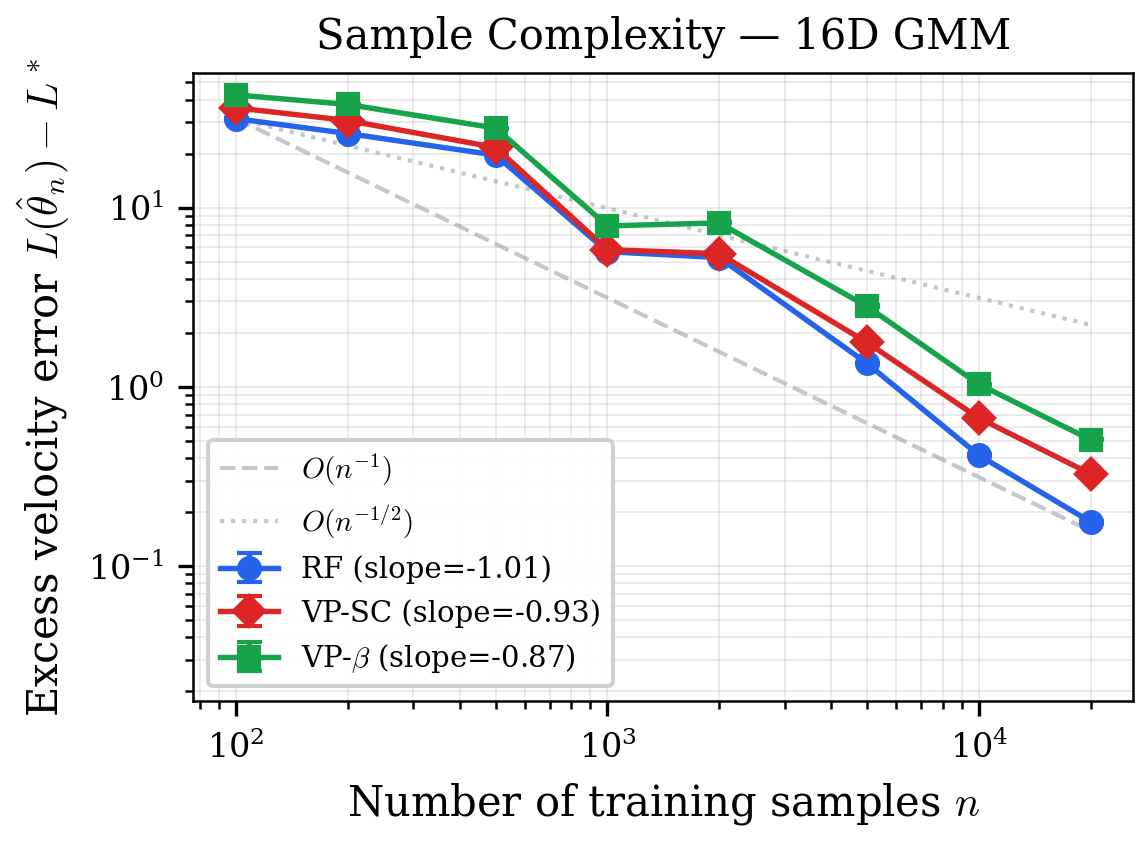}
        \caption{16D GMM}
    \end{subfigure}
    \begin{subfigure}[b]{0.42\textwidth}
        \includegraphics[width=\textwidth]{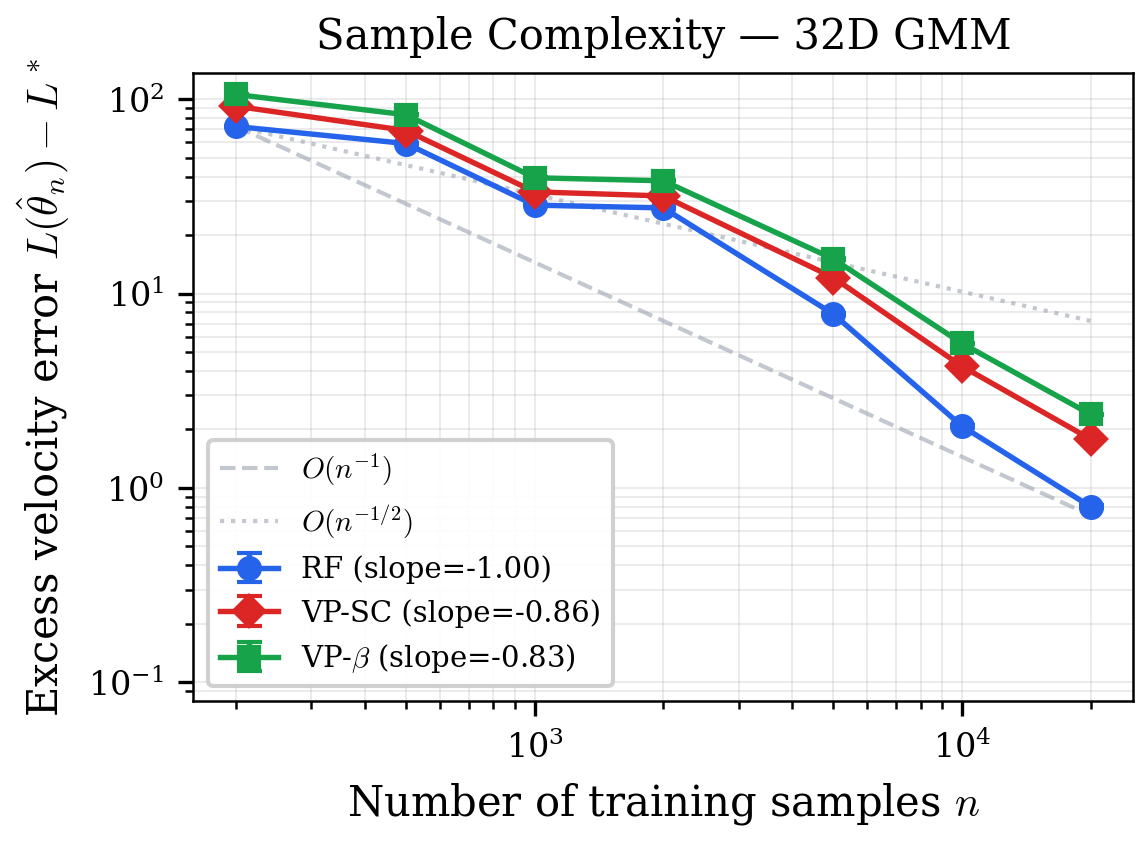}
        \caption{32D GMM}
    \end{subfigure}
    \caption{Excess velocity error vs.\ training-set size $n$ for four
    targets of increasing dimension, same architecture (4-layer MLP,
    width $512$). Error bars are $\pm 1$~SEM over $7$ seeds.}
    \vspace{-2mm}
    \label{fig:rate-main}
\end{figure}

%% file: conclusion.tex
\newpage
\section{Conclusion}

We show that flow matching achieves $\widetilde{O}(\varepsilon^{-2})$ sample complexity for the statistical term under ERM access, with the Wasserstein guarantee $W_2 \le K(\varepsilon + C\sqrt{\eapp})$ making the approximation error explicit. This improves the $\varepsilon$-scaling from $\widetilde{O}(\varepsilon^{-4})$ to $\widetilde{O}(\varepsilon^{-2})$ and removes the PL condition.

We emphasize that our result does not establish a statistical separation between different flow-matching variants. The fast statistical rate $\widetilde{O}(\Lloss^2 P/n)$ holds uniformly across schedules under the stated assumptions. What differs across schedules is not the rate itself, but the constants and the approximation floor $\eapp$. In particular, rectified-flow and optimal transport (OT) paths typically have smaller target-size constants $c_Z$, which can reduce the estimation constants. This provides a possible explanation for their empirical advantage.

An important direction for future work is to extend these guarantees beyond the current squared-loss setting and better understand when fast rates arise for general flow-matching formulations. In particular, it would be interesting to study whether iterative linearization, trajectory straightening, or related structural reductions can induce approximate Bernstein conditions and localized complexity behavior analogous to those established here.

%% file: appendix.tex
\newpage
\appendix

\section*{Appendix}

\section{Related Work}
\label{app:related}

Some of the key related works are summarized in \cref{tab:Table1}. The table compares the sample complexity of various generative models, including DDPMs (Denoising Diffusion Probabilistic Models), flow-matching models, and our result for general squared-loss flow matching. In these formulas, $\varepsilon$ denotes the desired approximation error, typically measured in Wasserstein distance. The sample complexity depends on $d$ (data dimensionality), $\Theta$ (bound on parameters in the neural network), $R$ (hypothesis class support radius), $D$ (depth of the neural network), $W$ (width of the neural network), and $P$ (number of parameters). In our setting, $\Lloss$ denotes the Lipschitz constant of the pointwise loss (schedule-dependent; see \cref{sec:schedules}), and $\eapp$ is the approximation error defined in \cref{ass:approx}. Our result achieves more efficient scaling with respect to $\varepsilon$ compared to previous methods and removes the PL assumption. The dependence on the ambient dimension $d$ enters through the loss bound, approximation error, and network complexity $P$, which for a fully connected network of depth $D$ and width $W$ scales as $\mathcal{O}(dW + DW^2)$.

Theoretical analysis of sample complexity for generative models has primarily focused on diffusion-based methods.
\cite{block2022generativemodelingdenoisingautoencoders} assume a hypothesis class $\cF$ of $\R^d$-valued functions that are $L^2$-Lipschitz, supported on a Euclidean ball of radius $R$, and sufficiently expressive to uniformly approximate the true score function $s_t$ on this ball.
\cite{chen2023samplingeasylearningscore} simplified the original problem by proving that learning the score function accurately enough is sufficient for accurately sampling from the target distribution.
Building on this result, \cite{gupta2025improvedsamplecomplexitybounds} showed a sample complexity of $O(\varepsilon^{-5})$ under the assumption of access to the empirical risk minimizer (ERM), while \cite{gaur2025improvedsamplecomplexitydiffusion} further improved this to $O(\varepsilon^{-4})$ for DDPMs. The latter (without ERM access) requires the PL condition.

Analogous to results for diffusion models \cite{chen2023samplingeasylearningscore}, prior work \cite{benton2024errorboundsflowmatching} reduces generative modeling to the problem of learning an unknown velocity field. Building on this perspective, \cite{gaur2025generativemodelingcontinuousflows} established $\widetilde{O}(\varepsilon^{-4})$ rates using global uniform convergence arguments and the PL condition. However PL does most of the heavy lifting in their proofs involving both statistical and optimization terms.

Rectified Flow was introduced by \cite{liu2022flowstraightfastlearning} with linear interpolation paths. \cite{lee2024improvingtrainingrectifiedflows} show that few rectification steps suffice for nearly straight trajectories. \cite{bansal2025wassersteinconvergencestraightnessrectified} study asymptotic Wasserstein convergence under accurate velocity estimation. Variance-preserving formulations \cite{song2021scorebasedgenerativemodelingstochastic} and optimal-transport paths \cite{lipman2023flowmatchinggenerativemodeling} are widely used alternatives. However, existing analyses do not address finite-sample complexity across these variants, nor do they exploit the common squared-loss structure, which is the focus of this work.

The connection between Bernstein-type variance conditions and fast rates in empirical risk minimization is well established \cite{bartlett2005local,koltchinskii2006local,bousquet2002bennett}. We extend this machinery to the \emph{approximate} Bernstein setting arising from squared-loss regression with model misspecification by carrying the additive residual $\rho$ through the standard peeling and Talagrand-inequality steps (\cref{thm:approx-local-rad}); the resulting oracle inequality may be of independent interest.

\paragraph{Approximation-based rates.}
\cite{fukumizu2025flow,oko2023diffusion,fu2024conditional} control network approximation and estimation under regularity assumptions. Our fixed-class analysis complements these end-to-end results by isolating the statistical term and retaining $\eapp$ explicitly. The comparison of available upper bounds does not establish a statistical separation between flow matching and DDPMs.

\paragraph{Misspecified least squares.}
The inequality in \cref{lem:approx-bernstein} follows from standard squared-loss identities. Offset-complexity analyses \cite{liang2015offset,kanade2024offset} and aggregation of empirical minimizers \cite{rakhlin2017entropy} provide alternative fast-rate approaches. We retain plain ERM and carry the misspecification residual through classical localization, yielding an additive approximation floor.

\section{Consequences of the Architectural Assumptions}
\label{app:ass_proved}

In this section, we verify that several regularity properties follow directly from \cref{ass:nn} and \cref{ass:schedule}. All results hold on the truncated event $\mathcal{E}_M := \{\|X_0\| \le M,\; \|X_1\| \le M\}$.

\begin{corollary}[Compactness of the parameter space]
\label{corr:compact_space}
The parameter space $\Theta$ is compact.
\end{corollary}

\begin{proof}
By \cref{ass:nn}, the total number of trainable parameters is finite, say $P < \infty$. Moreover, for each layer the corresponding weight matrix $W$ satisfies $\|W\|_{\mathrm{op}} \le V$. Since the operator norm dominates every entry, $|W_{ij}| \le \|W\|_{\mathrm{op}} \le V$, so each individual weight parameter is bounded in absolute value by $V$.

Hence, the parameter space $\Theta \subset \R^P$ is contained in a bounded set. Since the spectral-norm constraint $\|W\|_{\mathrm{op}} \le V$ defines a closed set, $\Theta$ is closed and bounded in finite-dimensional Euclidean space. By the Heine--Borel theorem, $\Theta$ is compact.
\end{proof}

\begin{corollary}[Uniform boundedness of $v_\theta$]
\label{corr:bounded_v}
The network output $v_\theta$ is uniformly bounded over $\Theta$ and $(x,t)$.
\end{corollary}

\begin{proof}
By \cref{ass:nn}, the output of every hidden layer is bounded in $[-b,b]$ component-wise, so a hidden-layer output $z \in \R^{m}$ (where $m$ denotes the relevant layer width) satisfies $\|z\|_2 \le \sqrt{m}\, b$. Let $\widetilde W$ denote the final affine map including its bias column, and augment $z$ by a constant coordinate. Since $\|\widetilde W\|_{\mathrm{op}}\le V$,
\begin{equation}
    \|v_\theta(x,t)\|_2=\|\widetilde W(z,1)\|_2\le V\sqrt{1+mb^2}.
\end{equation}
Hence there exists a constant $B_v < \infty$, depending only on $(b,V,m,L_\phi,D)$, such that
\begin{equation}
\sup_{\theta \in \Theta} \sup_{(x,t)} \|v_\theta(x,t)\| \le B_v.
\end{equation}
\end{proof}

\begin{corollary}[Differentiability and Lipschitz continuity in the input]
\label{corr:lipschitz_x}
For every $\theta \in \Theta$, the map $v_\theta(x,t)$ is differentiable in $(x,t)$ and Lipschitz continuous in its first argument.
\end{corollary}

\begin{proof}
Each layer of the network consists of an affine map followed by a differentiable activation function. Therefore, $v_\theta(x,t)$ is differentiable in $(x,t)$ by repeated application of the chain rule.

Furthermore, each activation function is $L_\phi$-Lipschitz and, by \cref{ass:nn}, each weight matrix $W$ satisfies $\|W\|_{\mathrm{op}} \le V$. Consequently, each hidden layer is $L_\phi V$-Lipschitz and the linear output layer is $V$-Lipschitz with respect to the $\ell_2$ norm. By composition, for any $x_1,x_2$,
\[
\|v_\theta(x_1,t) - v_\theta(x_2,t)\|
\le V(L_\phi V)^{D-1} \|x_1 - x_2\|,
\]
uniformly over $\theta \in \Theta$. Hence, $v_\theta$ is $L_t$-Lipschitz in its first argument with $L_t \le V(L_\phi V)^{D-1}$.
\end{proof}

\begin{corollary}[Lipschitz continuity in the parameters]
\label{corr:lipschitz_theta}
Restrict the input to $\|x\|\le R$ and $t\in[0,1]$. Let $W$ be the maximum hidden-layer width, and define
\[
U_R=\max\{\sqrt{R^2+2},\sqrt{1+Wb^2}\},\qquad
L_R=D U_R\max\{1,L_\phi V\}^{D-1}.
\]
Then
\[
\sup_{\|x\|\le R,\ t\in[0,1]}\|v_\theta(x,t)-v_{\theta'}(x,t)\|\le L_R\|\theta-\theta'\|_2.
\]
\end{corollary}

\begin{proof}
Include biases through a constant input coordinate. The augmented input has norm at most $\sqrt{R^2+2}$, and every augmented hidden activation has norm at most $\sqrt{1+Wb^2}$. Replace the weight matrices one layer at a time. A perturbation of matrix $j$ contributes at most
\[
U_R(L_\phi V)^{D-j}\|\widetilde W_j-\widetilde W_j'\|_{\mathrm F}
\]
to the output difference. Summing these contributions, bounding each propagation factor by $\max\{1,L_\phi V\}^{D-1}$, and using
$\sum_j\|\widetilde W_j-\widetilde W_j'\|_{\mathrm F}\le D\|\theta-\theta'\|_2$ proves the claim.
\end{proof}

\begin{corollary}[Lipschitz constant of the pointwise loss]
\label{corr:L_ell}
On the truncated event $\mathcal{E}_M$, the pointwise squared loss $\loss(v_\theta, Z_t) = \|v_\theta - Z_t\|^2$ is $\Lloss$-Lipschitz in $v_\theta$, with
\[
\Lloss = 2(B_v + c_Z(2M+1)).
\]
\end{corollary}

\begin{proof}
Using the identity $\|a\|^2 - \|b\|^2 = \langle a - b,\; a + b\rangle$ with $a = v_\theta - Z_t$ and $b = v_{\theta'} - Z_t$:
\begin{align}
|\|v_\theta - Z_t\|^2 - \|v_{\theta'} - Z_t\|^2|
&= |\langle v_\theta - v_{\theta'},\; v_\theta + v_{\theta'} - 2Z_t \rangle| \nonumber \\
&\le \|v_\theta - v_{\theta'}\| \cdot (\|v_\theta\| + \|v_{\theta'}\| + 2\|Z_t\|) \nonumber \\
&\le \|v_\theta - v_{\theta'}\| \cdot (2B_v + 2\|Z_t\|).
\end{align}
On $\mathcal{E}_M$, by \cref{ass:schedule},
\[
\|Z_t\| \le c_Z(\|X_0\| + \|X_1\| + 1) \le c_Z(2M + 1).
\]
Therefore $\Lloss = 2B_v + 2c_Z(2M+1) = 2(B_v + c_Z(2M+1))$.
\end{proof}

\section{Proof of Statistical Error Bound}
\label{app:statistical}

In this section, we present the complete proof of \cref{thm:stat-main}. The argument proceeds through: covering numbers (\cref{app:sec:covering}), the local Rademacher framework (\cref{app:sec:local-rad}), the approximate Bernstein condition and its consequences (\cref{app:sec:ab-full}), the Dudley entropy integral (\cref{app:sec:dudley}), the fixed-point computation (\cref{app:sec:fixedpoint}), the bounded-loss bound (\cref{app:sec:bounded}), and the sub-Gaussian extension (\cref{app:sec:subgauss}).

\subsection{Covering Numbers for Neural Network Classes}
\label{app:sec:covering}

We use covers of the vector-valued network output, with its Euclidean norm. On $\|x\|\le R$, $t\in[0,1]$, \cref{corr:lipschitz_theta} gives
\[
\|v_\theta-v_{\theta'}\|_\infty\le L_R\|\theta-\theta'\|_2.
\]
The parameter set lies in the Euclidean ball of radius $V\sqrt P$. A maximal $\delta$-separated subset of this set is an internal $\delta$-cover. Comparing the disjoint balls of radius $\delta/2$ around its points with the ball of radius $V\sqrt P+\delta/2$ gives
\[
\mathcal N(\delta,\Theta,\|\cdot\|_2)\le(1+2V\sqrt P/\delta)^P.
\]
Consequently,
\begin{equation}\label{eq:app-cover-log}
\log\mathcal N_\infty(u,\mathcal F)\le P\log(1+2V\sqrt P L_R/u).
\end{equation}
There is no additional covering factor for the output coordinates: the parameter Lipschitz bound already uses the vector norm.

For the loss class $\mathcal H=\{h_\theta\}$,
\[
\|h_\theta-h_{\theta'}\|_\infty\le\Lloss L_R\|\theta-\theta'\|_2.
\]
This also holds for $h_\theta=g_\theta\mathbf1_{\mathcal E_M}$, because the common indicator decreases distances. An internal cover of a subclass can be obtained from a cover of the full class at half the radius, by replacing each occupied ball's center by a point of that subclass. Therefore, with
\begin{equation}\label{eq:app-cover-A}
A_R=B_\ell+4V\sqrt P\,\Lloss L_R,
\end{equation}
every subclass $\mathcal H'\subseteq\mathcal H$ satisfies
\begin{equation}\label{eq:app-cover-L2}
\log\mathcal N(u,\mathcal H',L_2(P_n))\le P\log(1+A_R/u).
\end{equation}

\subsection{Local Rademacher Complexity Framework}
\label{app:sec:local-rad}

We follow the localization framework of \cite{bartlett2005local}. For a class $\mathcal H$, use the absolute Rademacher average
\[
R_n(\mathcal H)=\E_\epsilon\sup_{h\in\mathcal H}\left|\frac1n\sum_{i=1}^n\epsilon_i h(\xi_i)\right|.
\]

\begin{definition}[Sub-root function]
A function $\psi:[0,\infty)\to[0,\infty)$ is sub-root if it is nonnegative, nondecreasing, and $r\mapsto\psi(r)/\sqrt r$ is nonincreasing for $r>0$. A nonzero sub-root function has a unique positive fixed point.
\end{definition}

We apply this framework to the excess-loss second-moment classes $\{h:Ph^2\le r\}$. The approximate Bernstein residual and the population discrepancy introduced by truncation are retained explicitly in \cref{app:sec:shift}.

\subsection{Approximate Bernstein Condition}
\label{app:sec:ab-full}

Applying Young's inequality with parameter $\eta\in(0,1)$ to \eqref{eq:before-young} gives
\begin{equation}\label{eq:app-general-eta}
\E\|v_\theta-v_{\theta^*}\|^2\le\frac{1}{1-\eta}\bigl(\cL(\theta)-\cL(\theta^*)\bigr)+\frac{1}{\eta(1-\eta)}\eapp.
\end{equation}
Choosing $\eta=1/2$:
\begin{equation}\label{eq:app-ab-final}
\boxed{\E\|v_\theta(X_t,t)-v_{\theta^*}(X_t,t)\|^2\le2\bigl(\cL(\theta)-\cL(\theta^*)\bigr)+4\eapp.}
\end{equation}
This is the approximate Bernstein condition (\cref{lem:approx-bernstein}).

To obtain variance control, define the pointwise excess loss
\[
g_\theta(\xi):=\|v_\theta(X_t,t)-Z_t\|^2-\|v_{\theta^*}(X_t,t)-Z_t\|^2,
\]
so that $\Delta_\theta:=Pg_\theta=\cL(\theta)-\cL(\theta^*)\ge0$. For bounded targets, use $h_\theta=g_\theta$. For sub-Gaussian data, use $h_\theta=g_\theta\mathbf1_{\mathcal E_M}$; all expectations remain under the original sampling law. If $\|Z_t\|\le M'$ on the domain where $h_\theta$ is nonzero, then $\Lloss=2(B_v+M')$ gives
\[
|g_\theta|\le\Lloss\|v_\theta(X_t,t)-v_{\theta^*}(X_t,t)\|
\]
on that domain. Therefore
\begin{equation}\label{eq:app-var-l2}
\Var(h_\theta)\le Ph_\theta^2\le\Lloss^2\E\|v_\theta-v_{\theta^*}\|^2.
\end{equation}
Using \eqref{eq:app-ab-final}:
\begin{equation}\label{eq:app-ab-variance}
\boxed{Ph_\theta^2\le2\Lloss^2\Delta_\theta+4\Lloss^2\eapp=:B_0\Delta_\theta+\rho,}
\end{equation}
where $B_0:=2\Lloss^2$ and $\rho:=4\Lloss^2\eapp$.

\subsection{Localized Rademacher Theorem under Approximate Bernstein}
\label{app:sec:shift}

\begin{theorem}
\label{thm:approx-local-rad-app}
Let $\mathcal H=\{h_\theta:\theta\in\Theta\}$ contain zero, with $|h_\theta|\le B_\ell$, $B_0,B_\ell>0$, and $\rho,\tau\ge0$. Suppose
\[
\Delta_\theta\ge0,\qquad |Ph_\theta-\Delta_\theta|\le\tau,\qquad Ph_\theta^2\le B_0\Delta_\theta+\rho.
\]
Let $\psi$ be a sub-root function with positive fixed point $r^*$ such that
\[
\psi(r)\ge B_0\E R_n\{h\in\mathcal H:Ph^2\le r\}.
\]
Set
\[
J=\left\lceil\log_2\max\{1,B_\ell^2/r^*\}\right\rceil,\qquad y=x+\log(J+1).
\]
For $K>1$ and $x>0$, with probability at least $1-e^{-x}$, every $\theta$ satisfies
\begin{align}\label{eq:app-approx-erm}
\Delta_\theta
&\le\frac K{K-1}(P_nh_\theta+\tau)
+\frac{48K^2(B_0+B_\ell)r^*}{(K-1)B_0^2}\nonumber\\
&\quad+\frac{9K^2(B_0+B_\ell)y}{(K-1)n}
+\frac{K\rho}{2(K-1)B_0}.
\end{align}
\end{theorem}

\begin{proof}
Write $\mathcal H(r)=\{h\in\mathcal H:Ph^2\le r\}$ and $Q=B_0+B_\ell$. The following $J+1$ shells cover the class:
\[
\mathcal H_-=\mathcal H(r^*),\qquad
\mathcal H_k=\{h:2^{k-1}r^*<Ph^2\le2^kr^*\},\quad k=1,\ldots,J.
\]
For $R\ge r^*$, the sub-root property gives $\psi(R)\le\sqrt{Rr^*}$. By \cite[Theorem~2.1]{bartlett2005local}, with its parameter $\alpha=1$, with failure probability at most $e^{-y}$,
\begin{align}\label{eq:app-bousquet-shell}
\sup_{h\in\mathcal H(R)}(Ph-P_nh)
&\le4\E R_n(\mathcal H(R))+\sqrt{\frac{2Ry}{n}}+\frac{8B_\ell y}{3n}\\
&\le\frac{4\sqrt{Rr^*}}{B_0}+\sqrt{\frac{2Ry}{n}}+\frac{8B_\ell y}{3n}.\nonumber
\end{align}
Apply this to $R=r^*,2r^*,\ldots,2^Jr^*$. A union bound gives simultaneous validity with probability at least $1-(J+1)e^{-y}=1-e^{-x}$.

For $h_\theta\in\mathcal H_k$, the lower shell endpoint implies
\begin{equation}\label{eq:app-shell-Eg-bound}
2^kr^*\le2B_0\Delta_\theta+2\rho\le2Q\Delta_\theta+2\rho.
\end{equation}
Using $\Delta_\theta\le Ph_\theta+\tau$ and splitting square roots,
\begin{align}\label{eq:app-after-split}
\Delta_\theta-P_nh_\theta
&\le\tau+\frac{4\sqrt{2Q\Delta_\theta r^*}}{B_0}
+2\sqrt{\frac{Q\Delta_\theta y}{n}}\nonumber\\
&\quad+\frac{4\sqrt{2\rho r^*}}{B_0}
+2\sqrt{\frac{\rho y}{n}}+\frac{8B_\ell y}{3n}.
\end{align}
Young's inequality gives
\[
\frac{4\sqrt{2Q\Delta_\theta r^*}}{B_0}
\le\frac{\Delta_\theta}{2K}+\frac{16KQr^*}{B_0^2},\qquad
2\sqrt{\frac{Q\Delta_\theta y}{n}}
\le\frac{\Delta_\theta}{2K}+\frac{2KQy}{n}.
\]
For the residual terms, use
\[
\frac{4\sqrt{2\rho r^*}}{B_0}\le\frac{32r^*}{B_0}+\frac{\rho}{4B_0},\qquad
2\sqrt{\frac{\rho y}{n}}\le\frac{4B_0y}{n}+\frac{\rho}{4B_0}.
\]
Since $Q\ge B_0,B_\ell$ and $K>1$, collecting terms gives
\begin{equation}\label{eq:app-pre-erm}
\Delta_\theta-P_nh_\theta\le\tau+\frac{\Delta_\theta}{K}+\frac{48KQr^*}{B_0^2}+\frac{9KQy}{n}+\frac{\rho}{2B_0}.
\end{equation}
For the small shell, use $\sqrt{2r^*y/n}\le r^*/B_0+B_0y/(2n)$ to obtain
\[
\Delta_\theta-P_nh_\theta\le\tau+\frac{5r^*}{B_0}+\frac{(B_0/2+8B_\ell/3)y}{n}.
\]
This is also dominated by the right side of \eqref{eq:app-pre-erm}, since $\Delta_\theta\ge0$. Dividing by $1-1/K=(K-1)/K$ proves the theorem.
\end{proof}

For ordinary excess losses, $h_\theta=g_\theta$, $\Delta_\theta=Pg_\theta$, and $\tau=0$. For masked losses, $Ph_\theta$ need not be nonnegative; the proof only requires $\Delta_\theta\ge0$. The same proof also covers the ordinary variance condition $\operatorname{Var}(g_\theta)\le B_0Pg_\theta+\rho$, because
\[
Pg_\theta^2\le(B_0+B_\ell)Pg_\theta+\rho,
\]
which is the relaxed second-moment bound used in \eqref{eq:app-shell-Eg-bound}.

\subsection{Dudley Entropy Integral}
\label{app:sec:dudley}

We bound the precise localized class required by the theorem:
\begin{equation}\label{eq:app-Fr-def}
\mathcal H(r)=\{h_\theta:Ph_\theta^2\le r\}.
\end{equation}
Here $h_\theta=g_\theta$ in the bounded case and $h_\theta=g_\theta\mathbf1_{\mathcal E_M}$ in the sub-Gaussian case. Let $R$ bound the network input on the domain where $h_\theta$ is nonzero, and define
\[
A_R=B_\ell+4V\sqrt P\,\Lloss L_R,\qquad H_R=1+P\log(1+A_Rn/B_\ell),\qquad q_R=H_R/n.
\]
By \eqref{eq:app-cover-L2}, every localized subclass has
\[
\log\mathcal N(u,\mathcal H(r),L_2(P_n))\le P\log(1+A_R/u).
\]
Set
\[
\mathfrak R(r)=\E R_n(\mathcal H(r)),\qquad D_n(r)^2=\sup_{h\in\mathcal H(r)}P_nh^2,\qquad\alpha=B_\ell/n.
\]
Conditionally on the observations, apply the empirical Dudley inequality \cite{van1996weak} to $\mathcal H(r)\cup(-\mathcal H(r))$. When $2D_n(r)>\alpha$,
\begin{align}\label{eq:app-dudley}
R_n(\mathcal H(r))
&\le C\alpha+\frac C{\sqrt n}\int_\alpha^{2D_n(r)}\sqrt{\log2+P\log(1+A_R/u)}\,du\\
&\le C\alpha+C\sqrt{q_R}\,D_n(r).\nonumber
\end{align}
If $2D_n(r)\le\alpha$, the direct bound $R_n(\mathcal H(r))\le D_n(r)\le\alpha$ gives the same conclusion. The cutoff is deterministic and the endpoint is the empirical radius.

Symmetrization and contraction for $u\mapsto u^2$, which is $2B_\ell$-Lipschitz on $[-B_\ell,B_\ell]$, give
\begin{align}\label{eq:app-radius-transfer}
\E D_n(r)^2
&\le r+\E\sup_{h\in\mathcal H(r)}|(P_n-P)h^2|\\
&\le r+2\E\E_\epsilon\sup_{h\in\mathcal H(r)}\left|\frac1n\sum_i\epsilon_i h(\xi_i)^2\right|\nonumber\\
&\le r+8B_\ell\mathfrak R(r).\nonumber
\end{align}
Taking expectations in \eqref{eq:app-dudley} and using Jensen,
\[
\mathfrak R(r)\le C B_\ell/n+C\sqrt{q_R}\sqrt{r+8B_\ell\mathfrak R(r)}.
\]
Split the square root and use Young's inequality on the term proportional to $\sqrt{B_\ell q_R\mathfrak R(r)}$. Since $q_R\ge1/n$, this yields
\begin{equation}\label{eq:app-rad-final}
\mathfrak R(r)\le C\bigl(\sqrt{rq_R}+B_\ell q_R\bigr).
\end{equation}

\subsection{Fixed-Point Computation}
\label{app:sec:fixedpoint}

A valid majorant from \eqref{eq:app-rad-final} is
\begin{equation}\label{eq:app-subroot}
\psi(r)=CB_0\bigl(\sqrt{rq_R}+B_\ell q_R\bigr)+B_0^2/n.
\end{equation}
It is nonnegative and nondecreasing, while
\[
\frac{\psi(r)}{\sqrt r}=CB_0\sqrt{q_R}+\frac{CB_0B_\ell q_R+B_0^2/n}{\sqrt r}
\]
is nonincreasing. Its positive fixed point obeys $r^*\ge B_0^2/n$ directly by construction. Writing $a=CB_0\sqrt{q_R}$ and $b_*=CB_0B_\ell q_R+B_0^2/n$, we have
\[
r^*=a\sqrt{r^*}+b_*\le r^*/2+a^2/2+b_*.
\]
Therefore
\begin{equation}\label{eq:app-rstar}
\frac{B_0^2}{n}\le r^*\le C'(B_0^2+B_0B_\ell)\frac{H_R}{n}.
\end{equation}
For squared loss, $B_\ell=\Lloss^2/4=B_0/8$, so
\begin{equation}\label{eq:app-rstar-over-B0}
r^*/B_0\le C\Lloss^2H_R/n.
\end{equation}
Also $B_\ell^2/r^*\le n/64$, hence
\[
J\le\lceil\log_2\max\{1,n\}\rceil,\qquad\log(J+1)\le C\log(1+n)\le C H_R.
\]
The finite-shell confidence cost is thus absorbed by $H_R$; the separate confidence parameter $x$ remains explicit.

\subsection{Bounded-Loss Bound}
\label{app:sec:bounded}

\begin{lemma}[Bounded-data excess risk]
\label{lem:bounded-app}
Suppose $\|X_t\|\le R$ and $\|Z_t\|\le M'$ almost surely. Under \cref{ass:nn,ass:approx,ass:erm}, for every $x>0$, with probability at least $1-e^{-x}$,
\begin{equation}\label{eq:app-excess-bounded}
\cL(\hat\theta)-\cL(\theta^*)\le C\Lloss^2(H_R+x)/n+2\eapp.
\end{equation}
\end{lemma}

\begin{proof}
Use $h=g$, $\tau=0$, and $K=2$ in \cref{thm:approx-local-rad-app}. The approximation term is exactly
\[
\frac{K\rho}{2(K-1)B_0}=\frac\rho{B_0}=2\eapp.
\]
The other terms follow from \eqref{eq:app-rstar-over-B0} and the preceding finite-shell estimate.
\end{proof}

Applying \cref{lem:approx-bernstein} gives
\begin{equation}\label{eq:app-stat-vel-bounded}
\E\|v_{\hat\theta}-v_{\theta^*}\|^2\le C\Lloss^2(H_R+x)/n+8\eapp.
\end{equation}

\begin{remark}
Dimension enters through $P$, the input and output bounds, and the approximation error $\eapp$. The covering contribution $H_R$ is logarithmic in the displayed parameter Lipschitz and weight bounds. Architecture and schedule quantities can also enter $B_v$, $\Lloss$, and the ODE stability factor outside logarithms. No dimension independence of $\eapp$ is assumed.
\end{remark}

\subsection{Extension to Sub-Gaussian Data}
\label{app:sec:subgauss}

The following argument applies to the original empirical risk minimizer $\hat\theta$ in \cref{ass:erm}. All population expectations below are under the original distribution of $\xi=(X_0,X_1,t)$.

Write $R=\max\{\|X_0\|,\|X_1\|\}$ and $a=4c\sigma^2$, where $c$ is the constant in \cref{ass:data}. A $1/2$-net of the unit sphere with at most $5^d$ elements satisfies $\|z\|\le2\max_u|\langle u,z\rangle|$. Applying the directional tail bound on this net and taking a union bound over $X_0,X_1$ gives
\begin{equation}\label{eq:app-norm-concentration}
\Pr(R>M)\le4\cdot5^d\exp(-M^2/a).
\end{equation}
For $n\ge2$ and $x>0$, choose
\begin{equation}\label{eq:app-truncation-radius}
M^2=a\bigl[d\log5+x+\log(4n^2)\bigr],\qquad p_M:=e^{-x}/n^2.
\end{equation}
Let $\mathcal E_M=\{R\le M\}$ denote the single-observation event, and define the sample event
\[
\mathcal A_M:=\bigcap_{i=1}^n\{\|X_0^i\|\le M,\ \|X_1^i\|\le M\}.
\]
Then $\Pr(\mathcal E_M^c)\le p_M$ and
\begin{equation}\label{eq:app-all-samples-truncated}
\Pr(\mathcal A_M^c)\le n p_M=e^{-x}/n\le e^{-x}.
\end{equation}

We next bound the population tail uniformly over the class. Put $H=(B_v+\|Z_t\|)^2$ and $\tau_M=\E[H\mathbf1_{\mathcal E_M^c}]$. Both squared losses $\ell_\theta$ and $\ell_{\theta^*}$ lie in $[0,H]$, so $|\ell_\theta-\ell_{\theta^*}|\le H$. Tail integration in \eqref{eq:app-norm-concentration} gives
\[
\E[R^2\mathbf1_{\{R>M\}}]
=M^2\Pr(R>M)+\int_{M^2}^{\infty}\Pr(R^2>s)\,ds\le(M^2+a)p_M.
\]
By \cref{ass:schedule}, $H\le2(B_v+c_Z)^2+8c_Z^2R^2$. Consequently,
\begin{align}\label{eq:app-uniform-tail}
\tau_M
&\le\bigl[2(B_v+c_Z)^2+8c_Z^2(M^2+a)\bigr]p_M\\
&\le\Lloss^2p_M,\qquad\Lloss=2\bigl(B_v+c_Z(2M+1)\bigr),\nonumber
\end{align}
where the last inequality uses $M^2\ge a$, valid for $d\ge1$. In particular, $\tau_M\le\Lloss^2 e^{-x}/n^2$.

Retain the original comparator $\theta^*$ and target $\Delta_\theta=\E g_\theta=\cL(\theta)-\cL(\theta^*)\ge0$, and use the bounded functions
\[
h_\theta=g_\theta\mathbf1_{\mathcal E_M}.
\]
Their means satisfy
\begin{equation}\label{eq:app-masked-target-bias}
|\E h_\theta-\Delta_\theta|\le\tau_M.
\end{equation}
Set $B_\ell=(B_v+c_Z(2M+1))^2=\Lloss^2/4$. The pointwise loss Lipschitz bound on $\mathcal E_M$, followed by the original-distribution approximate Bernstein inequality, gives
\begin{align}\label{eq:app-masked-second-moment}
|h_\theta|&\le B_\ell,\\
\E h_\theta^2&\le\Lloss^2\E\|v_\theta(X_t,t)-v_{\theta^*}(X_t,t)\|^2\nonumber\\
&\le2\Lloss^2\Delta_\theta+4\Lloss^2\eapp=B_0\Delta_\theta+\rho,\nonumber
\end{align}
where $B_0=2\Lloss^2$ and $\rho=4\Lloss^2\eapp$.
On $\mathcal E_M$, the network input has norm at most $R_M=c_I(2M+1)$. Multiplication by the common indicator preserves the covering bounds; hence \cref{app:sec:dudley,app:sec:fixedpoint} apply to $\{h_\theta\}$ with $R=R_M$, localizing by its own population second moment.

On $\mathcal A_M$, the original ERM inequality implies
\[
P_nh_{\hat\theta}=P_n(\ell_{\hat\theta}-\ell_{\theta^*})\le0.
\]
Apply \cref{thm:approx-local-rad-app} with $K=2$, $\tau=\tau_M$, and confidence parameter $x$. The concentration event fails with probability at most $e^{-x}$. Its intersection with $\mathcal A_M$ therefore has probability at least $1-2e^{-x}$, with no independence requirement. On this intersection,
\[
\Delta_{\hat\theta}\le C\Lloss^2(H_{R_M}+x)/n+2\eapp+2\tau_M.
\]
Using $\mathcal E^{\mathrm{stat}}\le2\Delta_{\hat\theta}+4\eapp$ and absorbing $4\tau_M$ in the statistical term gives
\begin{equation}\label{eq:app-stat-subgauss}
\mathcal E^{\mathrm{stat}}\le\frac{C\Lloss^2(H_{R_M}+x)}{n}+8\eapp.
\end{equation}
This proves \cref{thm:stat-main}.

The radius in \eqref{eq:app-truncation-radius} has the dependence $M=O(\sigma\sqrt{c(d+\log n+x)})$. Equivalently, for final confidence $1-\delta$, take $x=\log(2/\delta)$ and
\[
M^2=4c\sigma^2\bigl[d\log5+\log(8n^2/\delta)\bigr].
\]
Then $\Pr(\mathcal A_M^c)\le\delta/(2n)$ and $\tau_M\le\Lloss^2\delta/(2n^2)$. In particular,
\[
\Lloss^2=O\!\left(B_v^2+c_Z^2\left[1+c\sigma^2(d+\log n+\log(1/\delta))\right]\right).
\]
This displays the dimension dependence of truncation separately from the parameter count and other architecture constants.

\section{Optimization Error}
\label{app:opt}

Throughout the paper, \cref{ass:erm} gives $\mathcal E^{\mathrm{opt}}=0$.
For an approximate ERM $\theta\in\Theta$ with empirical loss gap at most
$\eta_{\mathrm{opt}}$, the localization argument can instead retain
this empirical tolerance, contributing $O(\eta_{\mathrm{opt}})$ to
the excess-risk bound. Thus, an independently established
$\eta_{\mathrm{opt}}$ no larger than the statistical term preserves
its order. This route uses the empirical loss gap directly and
does not require a parameter-distance estimate.

\paragraph{Conditional optimization guarantees.}
A PL condition, together with the applicable smoothness and step-size
conditions, can yield convergence of empirical loss. Overparameterized
network results provide another possible route, under their respective
architecture, data, initialization, and training assumptions
\cite{allen2019convergence,du2019gradient,liu2022loss}.
In particular, \cite{allen2019convergence} analyzes ReLU networks,
so its guarantee does not directly apply to the bounded, differentiable
activations in \cref{ass:nn}. Such conditions must be verified separately
before using these results to replace \cref{ass:erm}.
If the decomposition through $\mathcal E^{\mathrm{opt}}$ is used instead,
a separate population prediction-error bound of the required statistical
order is needed.

\section{Information-Theoretic Lower Bound}
\label{app:lower_bound}

We establish an $\Omega(\varepsilon^{-2})$ sample-complexity lower bound for integrated velocity estimation using a Gaussian location family. Related lower-bound examples for score estimation appear in \cite[Section~4]{gupta2025improvedsamplecomplexitybounds}.

\begin{lemma}[Lower bound for velocity estimation]
\label{lem:lower-bound}
Fix $\sigma>0$. Let $X_0\sim\mathcal N(0,\sigma^2)$ and $X_1\sim\mathcal N(\mu,\sigma^2)$ independently, where $|\mu|\le\sigma/2$, and let $p_{\mu,t}$ be the density of $X_t=(1-t)X_0+tX_1$. Write
\[
v_\mu(x,t):=\E_\mu[X_1-X_0\mid X_t=x].
\]
For any $n\ge1$, estimators based on $n$ independent training triples $(X_0^i,X_1^i,t_i)$, with $t_i\sim\mathrm{Uniform}[0,1]$, satisfy
\begin{equation}\label{eq:lb-main}
\inf_{\hat v}\sup_{|\mu|\le\sigma/2}
\E_\mu\!\left[\int_0^1\!\int_{\mathbb R}
|\hat v(x,t)-v_\mu(x,t)|^2p_{\mu,t}(x)\,dx\,dt\right]
\ge \frac{\sigma^2}{32n}.
\end{equation}
This family satisfies \cref{ass:data} with a uniform numerical tail constant.
\end{lemma}

\begin{proof}
Set $q_t=(1-t)^2+t^2$. Gaussian conditioning gives
\begin{equation}\label{eq:lb-velocity}
X_t\sim\mathcal N(t\mu,\sigma^2q_t),\qquad
v_\mu(x,t)=\frac{2t-1}{q_t}x+\frac{1-t}{q_t}\mu.
\end{equation}
Indeed, $\E[X_1-X_0]=\mu$ and $\operatorname{Cov}(X_1-X_0,X_t)=\sigma^2(2t-1)$. Moreover, for $|\mu|\le\sigma/2$, Gaussian tails imply
$\Pr(|X_j|\ge r)\le2\exp(-r^2/(8\sigma^2))$ for $j\in\{0,1\}$: for $r\ge\sigma$ use $r-|\mu|\ge r/2$, and for $r<\sigma$ the bound is trivial.

Choose $a=\sigma/(2\sqrt n)$ and the two hypotheses $\mu=\pm a$. Let $Q_+$ and $Q_-$ denote their full training-data laws. The source samples and times have the same law under both hypotheses, so
\[
\operatorname{KL}(Q_+\Vert Q_-)=\frac{2na^2}{\sigma^2}=\frac12,
\qquad
\operatorname{TV}(Q_+,Q_-)\le\frac12
\]
by Pinsker's inequality. For the test marginals, the same inequality gives
\[
\operatorname{TV}(p_{a,t},p_{-a,t})
\le\frac{ta}{\sigma\sqrt{q_t}}
\le\frac a\sigma\le\frac12.
\]
Define the common density $h_t(x):=\min\{p_{a,t}(x),p_{-a,t}(x)\}$, whose mass is at least $1/2$. Since $(1-t)/q_t\ge1$ on $[0,1/2]$, \eqref{eq:lb-velocity} yields
\begin{equation}\label{eq:lb-velocity-difference}
D^2:=\int_0^1\!\int_{\mathbb R}h_t(x)
|v_a(x,t)-v_{-a}(x,t)|^2\,dx\,dt
\ge\int_0^{1/2}4a^2\cdot\frac12\,dt=a^2.
\end{equation}

To account for the hypothesis-dependent loss measures, let $f_+$ and $f_-$ be densities of $Q_+$ and $Q_-$ on the training sample space, and let $R_+$ and $R_-$ denote the two expected losses in \eqref{eq:lb-main} for a fixed estimator. With $s$ denoting the training data, nonnegativity and the squared-distance inequality give
\begin{align*}
R_++R_-
&\ge\int\min\{f_+(s),f_-(s)\}\int_0^1\!\int_{\mathbb R}h_t(x)\\
&\qquad\times\Bigl(|\hat v_s(x,t)-v_a(x,t)|^2
+|\hat v_s(x,t)-v_{-a}(x,t)|^2\Bigr)\,dx\,dt\,ds\\
&\ge\frac{D^2}{2}\bigl(1-\operatorname{TV}(Q_+,Q_-)\bigr)
\ge\frac{a^2}{4}.
\end{align*}
Thus $\max\{R_+,R_-\}\ge a^2/8=\sigma^2/(32n)$. Taking the infimum over estimators proves the result. Independent estimator randomization can be included in $s$ without changing the argument.
\end{proof}

\begin{remark}
For fixed $\sigma$, achieving expected integrated squared velocity error at most $\varepsilon^2$ uniformly over this family requires $n=\Omega(\varepsilon^{-2})$. This matches, up to logarithms, the $\varepsilon$ exponent of the statistical velocity-estimation bound. The lower bound concerns estimation of the true velocity field, while the approximation error $\eapp$ depends on the chosen network class.
\end{remark}

\section{Numerical Experiments}
\label{app:experiments}

To examine the scaling predicted by Theorem~\ref{thm:main}, we trained the same
$4$-layer MLP on four data distributions of increasing dimension
$d \in \{2, 8, 16, 32\}$. The rectified-flow slopes are near $-1$,
consistent with the predicted $\widetilde{O}(L_\ell^2 P / n)$ scaling.
The SiLU experiments illustrate behavior beyond the bounded-activation
class used in the analysis.

For each task, schedule, and training-set size $n$, we train an MLP
of width $512$, depth $4$, with SiLU activations, using Adam with
cosine annealing and gradient clipping at $1.0$. Results are averaged
over $7$ random seeds, and error bars indicate $\pm 1$ standard error
of the mean.

To avoid bias from estimating $L^\ast$, we exploit the closed-form
Bayes-optimal velocity $v^\ast(x, t) = \mathbb{E}[Z_t \mid X_t = x]$
that exists when $X_1$ is a Gaussian mixture, since $X_t \mid
\text{component}$ is Gaussian under each schedule. We compute
$L^\ast$ to four significant figures by Monte Carlo with
$n_{\text{mc}} \ge 10^6$ samples, and evaluate the excess via the
paired identity $L(\hat\theta) - L^\ast = \mathbb{E}\|v_{\hat\theta}
- v^\ast\|^2$ (orthogonality of the regression residual). This lets
us measure excess down to $10^{-3}$ on the same draws.

Figure~\ref{fig:rate-verification} plots excess velocity error
against training-set size $n$ on log-log axes for four targets of
increasing dimension, all trained with the same MLP architecture:

\begin{itemize}
    \item \textbf{2D Gaussian.} A single full-rank Gaussian, for which
    $v^\ast$ is affine. All three schedules attain slope exactly
    $-1.00$.
    \item \textbf{8D GMM.} Eight axis-aligned modes at $\pm 5\, e_i$,
    $i=1,\dots,4$. Slopes: $-1.15$ (RF), $-1.07$ (VP-SC), $-1.09$
    (VP-$\beta$).
    \item \textbf{16D GMM.} Sixteen axis-aligned modes. Slopes:
    $-1.01$ (RF), $-0.93$ (VP-SC), $-0.87$ (VP-$\beta$).
    \item \textbf{32D GMM.} Thirty-two axis-aligned modes. Slopes:
    $-1.00$ (RF), $-0.86$ (VP-SC), $-0.83$ (VP-$\beta$).
\end{itemize}

The rectified-flow slope is at or near $-1$ from $d=2$ to $32$,
consistent with the scaling predicted by Theorem~\ref{thm:main}.
A nonzero approximation floor can contribute to the observed
slope-flattening of VP-SC and VP-$\beta$ at higher $d$;
a larger statistical prefactor alone does not change the predicted slope.

\paragraph{Schedule ordering and the $L_\ell$ analysis.}
In every panel of Figure~\ref{fig:rate-verification}, the curves are
vertically ordered RF~$<$~VP-SC~$<$~VP-$\beta$, and the gap between
schedules widens with $d$. This ordering is qualitatively consistent with the schedule-dependent
constant $L_\ell$ from Section~\ref{sec:schedules}, where $c_Z=1$
for rectified flow and $c_Z=\pi/2$ for VP-SC.
For VP-$\beta$, $\epsilon=10^{-3}$ is the experimental time cutoff;
the coefficient floor $\tau$ in Section~\ref{sec:schedules} is instead
determined by the schedule coefficients at the cutoff.
Since $L_\ell$ enters quadratically in the statistical-error prefactor,
a larger $c_Z$ contributes a larger prefactor. Empirically, at $d=2$
the three schedules sit within a factor of $1.3\times$ of one another
at every $n$, whereas at $d=32$ the gap between RF and VP-$\beta$
exceeds $2\times$.

\begin{figure}[h]
    \centering
    \begin{subfigure}[b]{0.245\textwidth}
        \includegraphics[width=\textwidth]{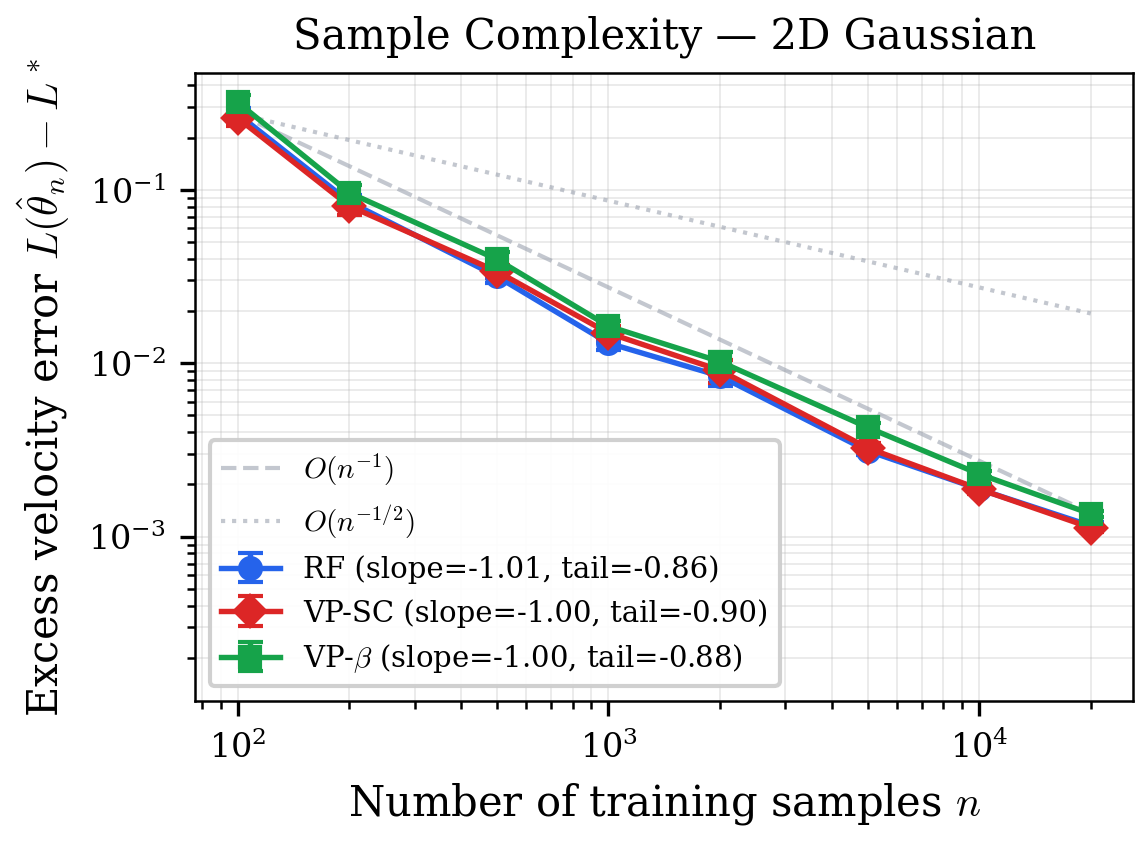}
        \caption{2D Gaussian}
    \end{subfigure}
    \begin{subfigure}[b]{0.245\textwidth}
        \includegraphics[width=\textwidth]{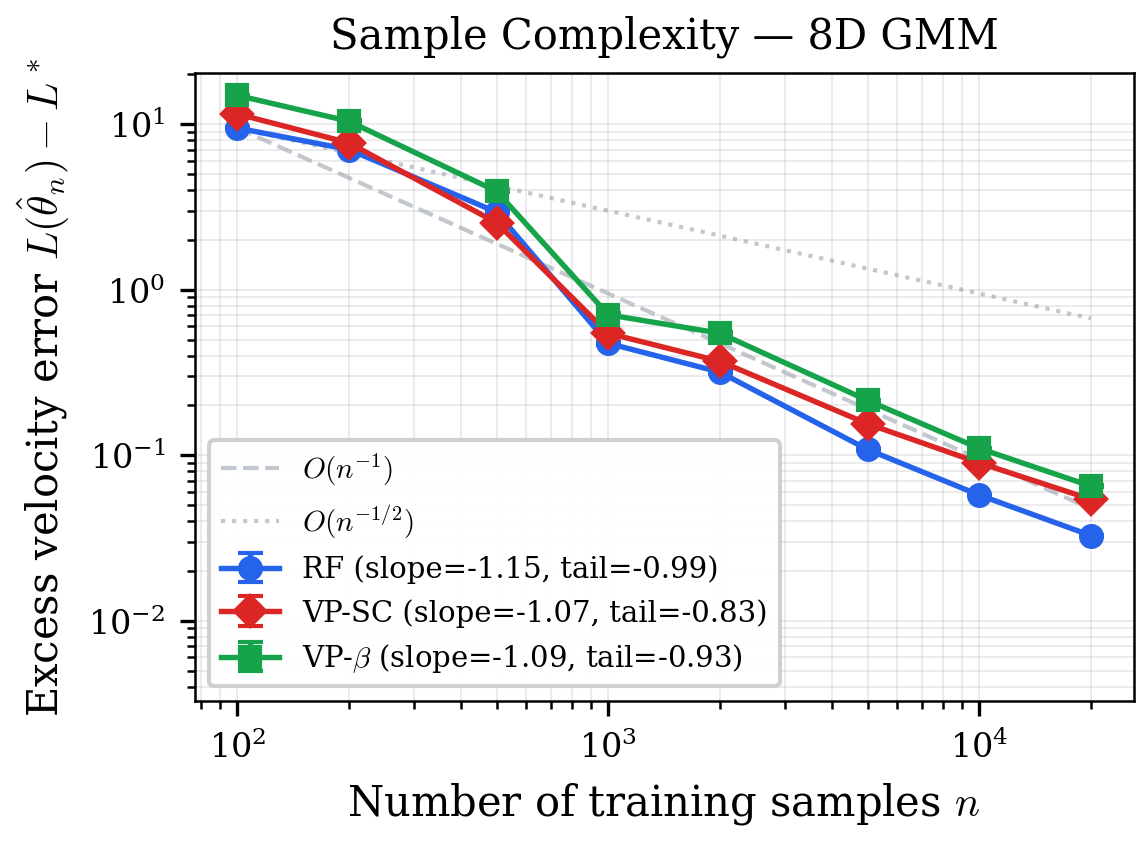}
        \caption{8D GMM}
    \end{subfigure}
    \begin{subfigure}[b]{0.245\textwidth}
        \includegraphics[width=\textwidth]{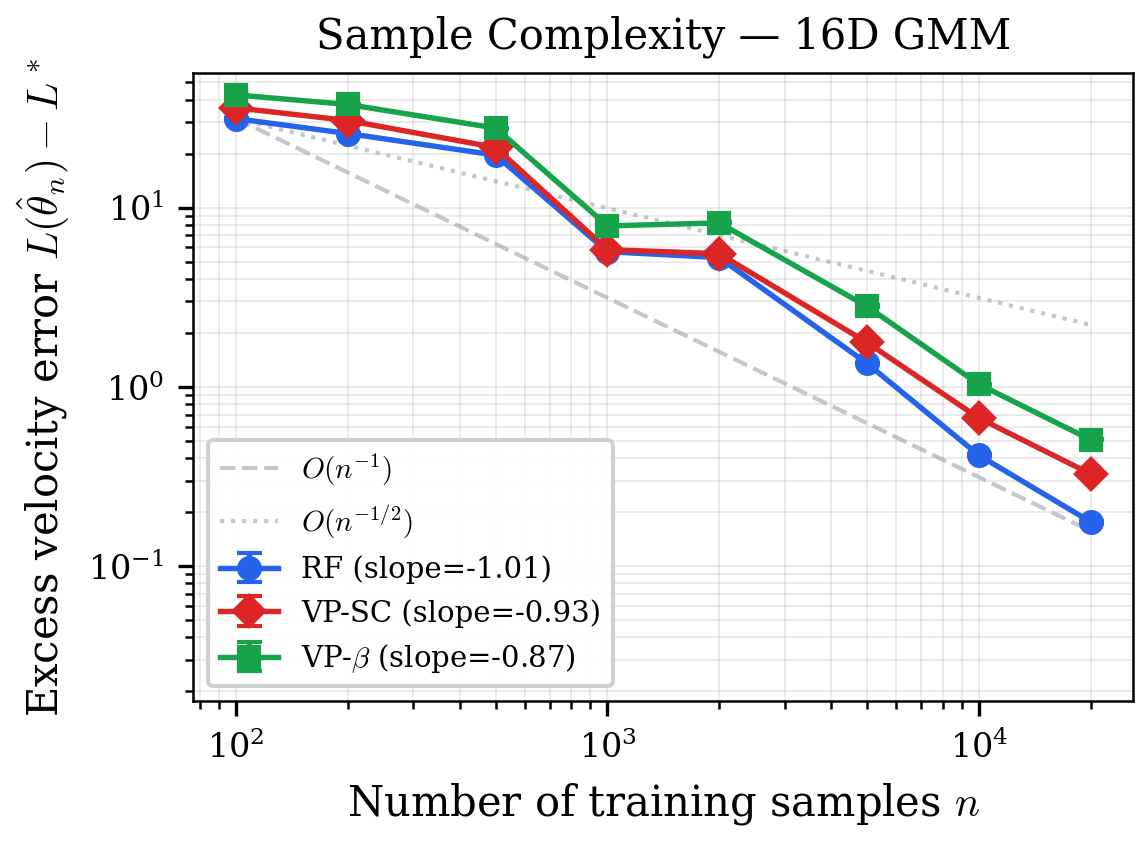}
        \caption{16D GMM}
    \end{subfigure}
    \begin{subfigure}[b]{0.245\textwidth}
        \includegraphics[width=\textwidth]{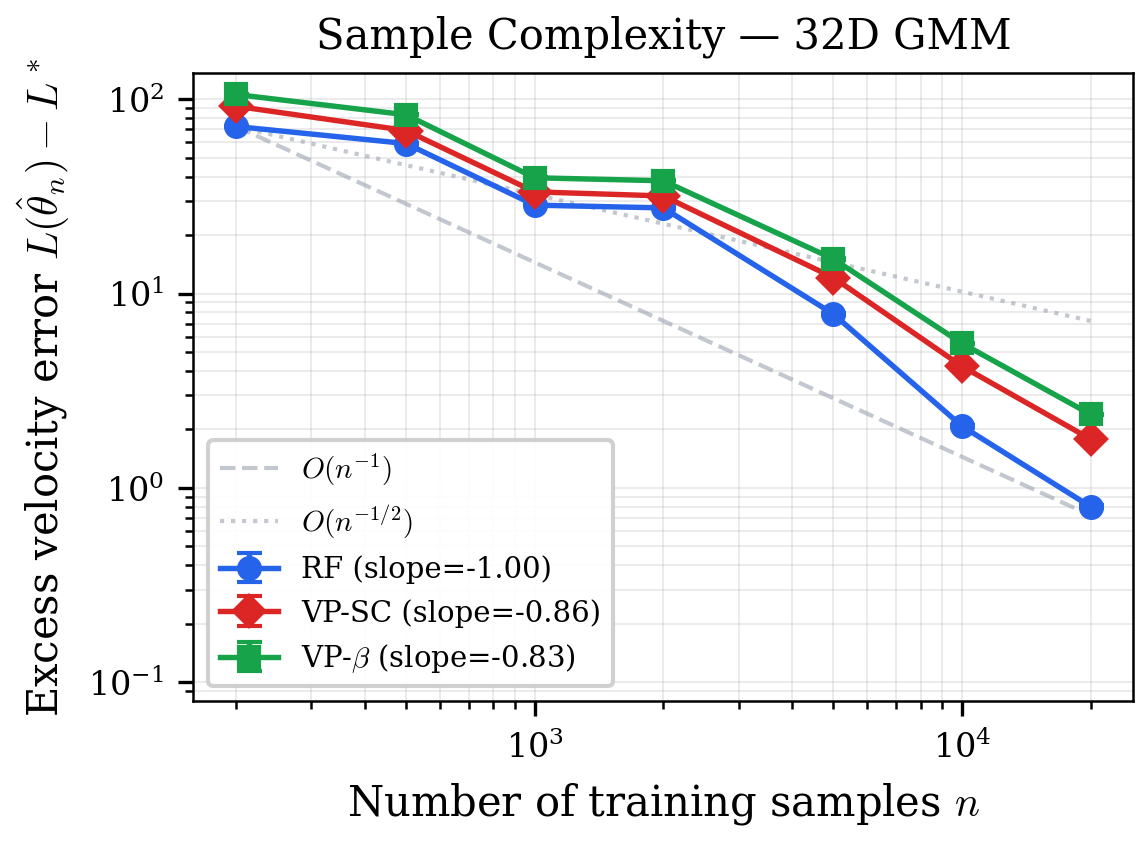}
        \caption{32D GMM}
    \end{subfigure}
    \caption{Excess velocity error $L(\hat\theta_n) - L^\ast$ versus
    training-set size $n$, on log-log axes, for four targets of
    increasing dimension trained with the same architecture (4-layer
    MLP, width $512$). Rectified flow attains a slope at or near
    $-1$ in every panel. The vertical ordering
    RF~$<$~VP-SC~$<$~VP-$\beta$ is qualitatively consistent with
    the $L_\ell$-dependent prefactor in
    Section~\ref{sec:schedules}. Reference $O(n^{-1})$ (dashed) and
    $O(n^{-1/2})$ (dotted) lines are included for visual reference.}
    \label{fig:rate-verification}
\end{figure}

\begin{remark}[Compute resources]
    All experiments were performed on a single NVIDIA RTX 2080 Ti (11~GB)
    with $4$ CPU threads, using PyTorch. Each rate experiment over a
    single dataset ($7$ seeds $\times$ $3$ schedules $\times$ $7$-$8$
    values of $n$ up to $n_{\max} = 20{,}000$) took approximately
    $3$-$4$ hours of GPU time.
\end{remark}

%% file: checklist.tex
\newpage
\section*{NeurIPS Paper Checklist}


\begin{enumerate}

\item {\bf Claims}
    \item[] Question: Do the main claims made in the abstract and introduction accurately reflect the paper's contributions and scope?
    \item[] Answer: \answerYes{} 
    \item[] Justification: Abstract/Introduction match Theorem 2

\item {\bf Limitations}
    \item[] Question: Does the paper discuss the limitations of the work performed by the authors?
    \item[] Answer: \answerYes{} 
    \item[] Justification: All discussed in remarks and conclusion

\item {\bf Theory assumptions and proofs}
    \item[] Question: For each theoretical result, does the paper provide the full set of assumptions and a complete (and correct) proof?
    \item[] Answer: \answerYes{} 
    \item[] Justification: Assumptions clearly stated and detailed proofs in appendix

    \item {\bf Experimental result reproducibility}
    \item[] Question: Does the paper fully disclose all the information needed to reproduce the main experimental results of the paper to the extent that it affects the main claims and/or conclusions of the paper (regardless of whether the code and data are provided or not)?
    \item[] Answer: \answerYes{} 
    \item[] Justification: Numerical Experiments section and supplemental material

\item {\bf Open access to data and code}
    \item[] Question: Does the paper provide open access to the data and code, with sufficient instructions to faithfully reproduce the main experimental results, as described in supplemental material?
    \item[] Answer: \answerYes{} 
    \item[] Justification: Code provided with readme file 

\item {\bf Experimental setting/details}
    \item[] Question: Does the paper specify all the training and test details (e.g., data splits, hyperparameters, how they were chosen, type of optimizer) necessary to understand the results?
    \item[] Answer: \answerYes{} 
    \item[] Justification: Numerical Experiments section

\item {\bf Experiment statistical significance}
    \item[] Question: Does the paper report error bars suitably and correctly defined or other appropriate information about the statistical significance of the experiments?
    \item[] Answer: \answerYes{} 
    \item[] Justification: Numerical Experiments section

\item {\bf Experiments compute resources}
    \item[] Question: For each experiment, does the paper provide sufficient information on the computer resources (type of compute workers, memory, time of execution) needed to reproduce the experiments?
    \item[] Answer: \answerYes{} 
    \item[] Justification: Single RTX 2080 and 4 CPU threads took 3-4 hours per dataset
    
\item {\bf Code of ethics}
    \item[] Question: Does the research conducted in the paper conform, in every respect, with the NeurIPS Code of Ethics \url{https://neurips.cc/public/EthicsGuidelines}?
    \item[] Answer: \answerYes{}{} 
    \item[] Justification: Theory paper with no human subjects

\item {\bf Broader impacts}
    \item[] Question: Does the paper discuss both potential positive societal impacts and negative societal impacts of the work performed?
    \item[] Answer: \answerNA{} 
    \item[] Justification: Theory work, no direct impact
    
\item {\bf Safeguards}
    \item[] Question: Does the paper describe safeguards that have been put in place for responsible release of data or models that have a high risk for misuse (e.g., pre-trained language models, image generators, or scraped datasets)?
    \item[] Answer: \answerNA{} 
    \item[] Justification: No such risks

\item {\bf Licenses for existing assets}
    \item[] Question: Are the creators or original owners of assets (e.g., code, data, models), used in the paper, properly credited and are the license and terms of use explicitly mentioned and properly respected?
    \item[] Answer: \answerNA{} 
    \item[] Justification: No assets used

\item {\bf New assets}
    \item[] Question: Are new assets introduced in the paper well documented and is the documentation provided alongside the assets?
    \item[] Answer: \answerNA{} 
    \item[] Justification: No new assets

\item {\bf Crowdsourcing and research with human subjects}
    \item[] Question: For crowdsourcing experiments and research with human subjects, does the paper include the full text of instructions given to participants and screenshots, if applicable, as well as details about compensation (if any)? 
    \item[] Answer: \answerNA{} 
    \item[] Justification: No human subjects

\item {\bf Institutional review board (IRB) approvals or equivalent for research with human subjects}
    \item[] Question: Does the paper describe potential risks incurred by study participants, whether such risks were disclosed to the subjects, and whether Institutional Review Board (IRB) approvals (or an equivalent approval/review based on the requirements of your country or institution) were obtained?
    \item[] Answer: \answerNA{} 
    \item[] Justification: No human subjects

\item {\bf Declaration of LLM usage}
    \item[] Question: Does the paper describe the usage of LLMs if it is an important, original, or non-standard component of the core methods in this research? Note that if the LLM is used only for writing, editing, or formatting purposes and does \emph{not} impact the core methodology, scientific rigor, or originality of the research, declaration is not required.
    \item[] Answer: \answerNA{} 
    \item[] Justification: No LLM in core method development

\end{enumerate}

\newpage